\documentclass{article} 
\usepackage{iclr2025_conference,times}
\iclrfinalcopy

\usepackage{amsmath,amsfonts,bm}

\def\eqref#1{equation~\ref{#1}}

\def\1{\bm{1}}

\DeclareMathAlphabet{\mathsfit}{\encodingdefault}{\sfdefault}{m}{sl}
\SetMathAlphabet{\mathsfit}{bold}{\encodingdefault}{\sfdefault}{bx}{n}

\usepackage{graphicx}
\usepackage[utf8]{inputenc} 
\usepackage[T1]{fontenc}    
\usepackage{hyperref}       
\usepackage{url}            
\usepackage{booktabs}       
\usepackage{amsfonts}       
\usepackage{nicefrac}       
\usepackage{microtype}      
\usepackage{xcolor}         
\usepackage{makecell}
\usepackage{booktabs}
\usepackage{tikz}
\usepackage{subcaption}
\usepackage{wrapfig}

\usepackage[linesnumbered,ruled]{algorithm2e}
\usepackage{amsmath,amssymb}
\usepackage{amsthm}
\usepackage{setspace}

\newtheorem{theorem}{Theorem}
\usetikzlibrary{decorations}
\usetikzlibrary{shapes}
\usetikzlibrary{shapes.symbols}
\usetikzlibrary{shapes.multipart}

\newcommand{\methodname}{COMiX}

\title{\methodname: Compositional Explanations using\\ Prototypes }

\author{%
  \makebox[\textwidth][c]{%
    Sarath Sivaprasad$^1$ \quad Dmitry Kangin$^2$ \quad Plamen Angelov$^2$ \quad Mario Fritz$^1$%
  }\\
  \makebox[\textwidth][c]{%
    $^1$CISPA Helmholtz Center for Information Security \quad $^2$LIRA Centre, Lancaster University, UK%
  } \\
  \makebox[\textwidth][c]{%
    \small{\texttt{\{sarath.sivaprasad, fritz\}@cispa.de}} \quad 
    \small{\texttt{\{d.kangin1, p.angelov\}@lancaster.ac.uk}}%
  }
}

\begin{document}

\maketitle

\begin{abstract}
Aligning machine representations with human understanding is key to improving interpretability of machine learning (ML) models. 
When classifying a new image, humans often explain their decisions by decomposing the image into concepts and pointing to corresponding regions in familiar images.
Current ML explanation techniques typically either trace decision-making processes to reference prototypes, generate attribution maps highlighting feature importance, or incorporate intermediate bottlenecks designed to align with human-interpretable concepts.
The proposed method, named \methodname, classifies an image by decomposing it into regions based on learned concepts and tracing each region to corresponding ones in images from the training dataset, assuring that explanations fully represent the actual decision-making process.
We dissect the test image into selected internal representations of a neural network to derive prototypical parts (primitives) and match them with the corresponding primitives derived from the training data. 
In a series of qualitative and quantitative experiments, we theoretically prove and demonstrate that our method, in contrast to \textit{post hoc} analysis, provides fidelity of explanations and shows that the efficiency is competitive with other inherently interpretable architectures. Notably, it shows substantial improvements in fidelity and sparsity metrics, including $48.82\%$ improvement in the C-insertion score on the ImageNet dataset over the best state-of-the-art baseline. 
\end{abstract}

\section{Introduction}

Neural networks (NNs) have been successfully applied across various computer vision tasks, achieving notable results in safety-critical domains such as medical image classification  \citep{huang2023self}, autonomous driving \citep{geiger2012we}, and robotics \citep{robinson2023robotic} amongst others. 
However, explaining their decisions remains an ongoing research challenge \citep{samek2021explaining}.

The two key factors in interpreting neural network decisions are: (1) representing the reasoning behind the prediction in human-understandable terms and (2) ensuring that the explanations accurately reflect the underlying computations of the neural network.
Beyond their face value, such interpretations can also help meet the legal requirements. 
The recently adopted EU AI Act \citep{EUAIAct} mandates that individuals should fully understand high-risk AI systems, enabling them to monitor these systems effectively, specifically requiring the ability to \textit{`correctly interpret the high-risk AI system's output'}.

Most existing explanation methods address this problem using attribution-based techniques, which highlight the parts of the input that contribute to a particular decision  \citep{selvaraju2017grad,chattopadhay2018grad,omeiza2019smooth}. However, these methods lack reliability as their explanations have been shown to be sensitive to factors which do not contribute to the model prediction \citep{kindermans2019reliability}.
To address this issue, concept- and prototype-based explanations have been proposed, which aim to link the decision to examples that illustrate the underlying concepts (\cite{kim2018interpretability,ghorbani2019towards,koh2020concept,tan2024post}). Nevertheless, such explanations have also been demonstrated to be insufficient for human understanding as they do not point to the reasons why the input is linked to the associated concept prototypes \citep{kim2016examples}. 

Studies of human understanding show that concepts can be decomposed into smaller constituents representing particular properties. These subconcepts can then be exemplified by the individual instances called \textit{prototypes} \citep{murphy2004big}. 
In this work, we propose a concept-based interpretable-by-design method, which highlights common class-defining features between the input image and the samples in the training dataset. This approach goes beyond attribution map predictions and presents a model, by design, that traces the decision to the original training data. Such decision-making process can be motivated by a number of safety-critical applications, for example, medical data analysis, where a doctor wants to find out the aspects that make this image similar to the previous ones.

\begin{figure}
\centering
    \includegraphics[width=0.75\linewidth]{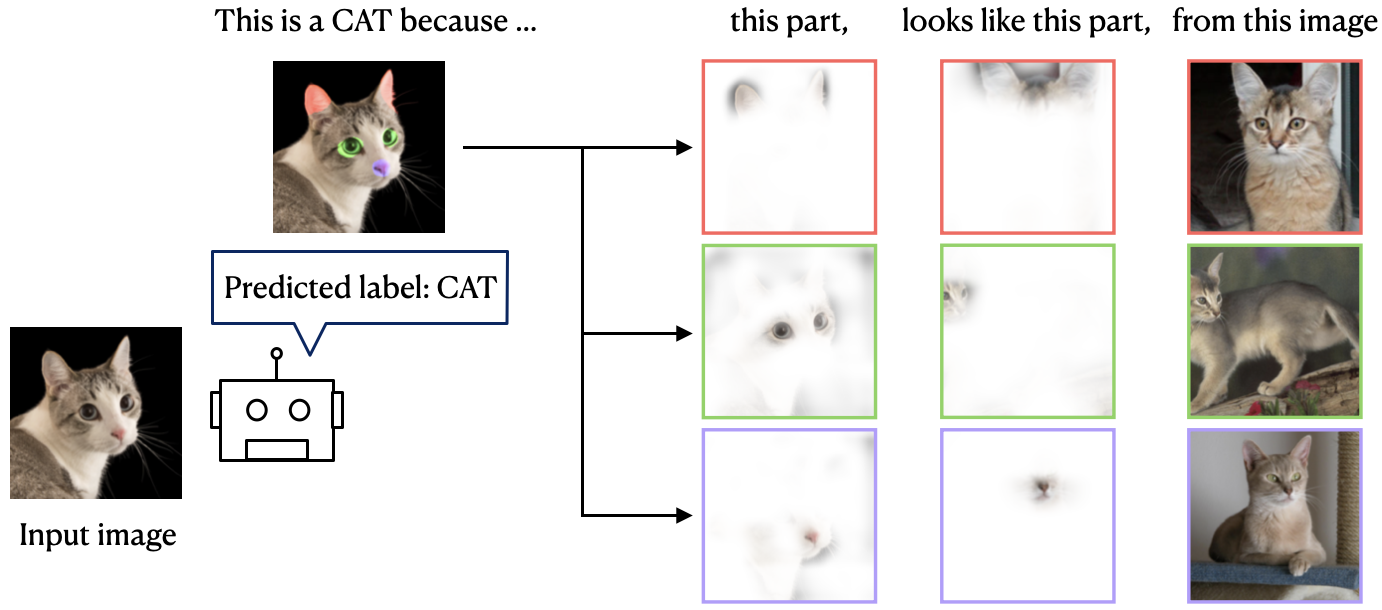}
    \caption{Humans often make sense of new or complex objects by comparing their parts to previously encountered prototypes (\cite{smith1974structure}). For example, when describing something unfamiliar, people tend to point out resemblances between parts of the new object and familiar prototypes by stating that ‘this part of the object looks like that other one I have seen before’. We propose a method to classify an image by decomposing it into regions based on learned concepts and tracing each region to the corresponding regions in images from training datasets. We refer to such interpretations as to \textit{'\methodname\ panels'}}
 \vspace{-0.4cm}
    \label{teaser}
\end{figure}

We illustrate the idea of the proposed method, called \methodname, in Figure \ref{teaser}. For every test sample, we predict the output by linking them to a set of features in the training data. This link, by design, provides interpretations through the relationship between the testing image and the samples from the training set. This idea also extends to counterfactual interpretations, which demonstrate how the test sample relates to the classes that the model did not predict. It can also address the diagnostics of the misclassification cases, attributing the misclassification to the training data conditioned on class-defining features. We follow the convention from \cite{rudin2019stop} which contrasts \textit{post hoc} explainability with \textit{ante hoc} interpretability. \methodname\ is not \textit{post hoc} and the interpretability comes from the decision-making. We formulate the following desiderata and demonstrate, in sections \ref{desiderata_demonstration} and \ref{eval_interpretability} how \methodname\ meets the demands of: 
 \vspace{-0.2cm}
\begin{itemize}
    \item \textbf{Fidelity}: The method should faithfully and wholly reflect the decision-making procedure, which is achieved by-design.  \vspace{-0.1cm}
    \item \textbf{Sparsity}: For meaningful interpretation, the given class should activate only a handful of concepts. We enforce sparsity by restricting the decision-making to class-defining features. We also measure sparsity in Section \ref{eval_interpretability} against the standard ViT \citep{dosovitskiy2021an} baseline. \vspace{-0.1cm}
    \item \textbf{Necessity}: The concept is important for making the decision and its presence in the input is necessary. We evaluate this using the causal matrices \citep{petsiuk2018rise} in Section~\ref{eval_interpretability}. \vspace{-0.1cm}
    \item \textbf{Sufficiency}: The concept presence in the input is sufficient for making the given decision. We present the proof in Section~~\ref{desiderata_demonstration}
\end{itemize}

The contributions of our paper are as follows:
\begin{itemize}
\item We propose a novel method, called \methodname, which reliably points prototypical regions in a testing image and matches them to regions in training images.
\item Based on this method, we demonstrate how this method can be built upon existing inherently interpretable architectures with an additional value of concept discovery.
\item We demonstrate, in a number of settings, the efficiency of \methodname\ through a series of qualitative and quantitative experiments, showing the advantages of the method over existing baselines in terms of fidelity and sparsity.
\end{itemize}

\section{Related work}

\paragraph {Explainable and interpretable AI.}
The early methods for neural network \textit{post hoc} explanations, such as the work by \cite{simonyan2013deep} and Grad-CAM (\cite{selvaraju2017grad}), were grounded in the idea of differentiating through the model. Other important backpropagation-based models include \cite{bach2015pixel,sundararajan2017axiomatic}. Perturbation-based methods, such as \cite{ribeiro2016should,lundberg2017unified,petsiuk2018rise,vstrumbelj2014explaining}, use perturbations to figure out input features' contributions. However, such a line of research is limited in its ability to capture the true inner workings of the original model \citep{rudin2019stop}. To address this concern, a number of by-design interpretable machine learning models have been proposed, presenting the interpretable architectures (\cite{bohle2022b,bohle2024b}), concept-bottleneck models (\cite{koh2020concept,shin2023closer,schrodi2024concept,losch2019interpretability,qian2022static}) and prototype-based interpretations \citep{chen2019looks,donnelly2022deformable,angelov2020towards}. 
\cite{fel2023craft} tackles a similar problem to the one in this paper: first, automatic extraction of concepts and then highlighting the similarities between such concepts and the testing image. However, the main conceptual difference between \cite{fel2023craft} and \methodname\ is that this work aims for by-design explanation of the decision-making while \cite{fel2023craft} addresses the problem of \textit{post hoc} analysis. %
In contrast to these works, the described method is both inherently interpretable and offers interpretation through the training data. 

\paragraph{Concept discovery.} 
Closely related to the studied problem interpretation is the challenge of concept discovery, motivated by the neuroscience studies in human reasoning \citep{bruner1957study}. \cite{kim2018interpretability} proposed a paradigm of concept activation vectors. Another study by \cite{ghorbani2019towards} proposes extracting visual concepts through segmentation.
Concept bottleneck models \citep{koh2020concept,shang2024incremental,sheth2024auxiliary,havasi2022addressing} introduce constraints into training so that the classifier is limited to using human-understandable features.
Similar to these models, \methodname\ also leverages concept discovery, where the concepts are individual interpretable classifier features.  On the contrary, we do not constrain the classifier to learn the human-understandable features and instead project the learned features into human-understandable space. In addition, \methodname\ traces these concepts back to the training data and provides inherent, by-design, interpretations, which have not, to the best of our knowledge, provided in the existing literature. 
ProtoPNet method (\cite{chen2019looks}) is a well-known baseline for concept discovery through patch prototypes. It has been further developed in a number of works such as \cite{donnelly2022deformable,ma2024looks,sacha2023protoseg,hase2019interpretable}. \cite{tan2024post} propose to combine \textit{post hoc} explainability methods with transparent concept-based reasoning. 
\cite{bontempelli2022concept} analyses the problem of attainment of confounders within ProtoPNet and addresses it with human-in-the-loop model debugging.

\paragraph{Evaluation of interpretability.}
\cite{hesse2023funnybirds} propose a synthetic dataset and benchmark for part-level analysis of explainable models for image classification. \cite{fel2024holistic} propose a set of metrics for explainable AI which assesses the quality of attribution-based explanations. They use the Insertion and Deletion metrics from \cite{petsiuk2018rise} for attribution assessment.
Important desiderata for concept extraction include sparsity of the outputs: not only do these outputs need to faithfully reflect the decision-making, but only a handful of concepts need to be activated for every testing image. To measure this ability, we leverage the metrics from the sparsity literature. \cite{diao2022pruning} propose a new PQ index metric, which measures the representation sparsity. 
One of the aspects, however, is that most of these metrics target the problem of attribution-based explanations. In our case, however, we combine concept-based and inherent attribution-based explanations, which allows us to evaluate the results using both C-insertion and C-deletion as well as the sparsity of concepts.

\section{Compositional explanations using \methodname}

\begin{figure}
\centering
  \includegraphics[width=\linewidth]{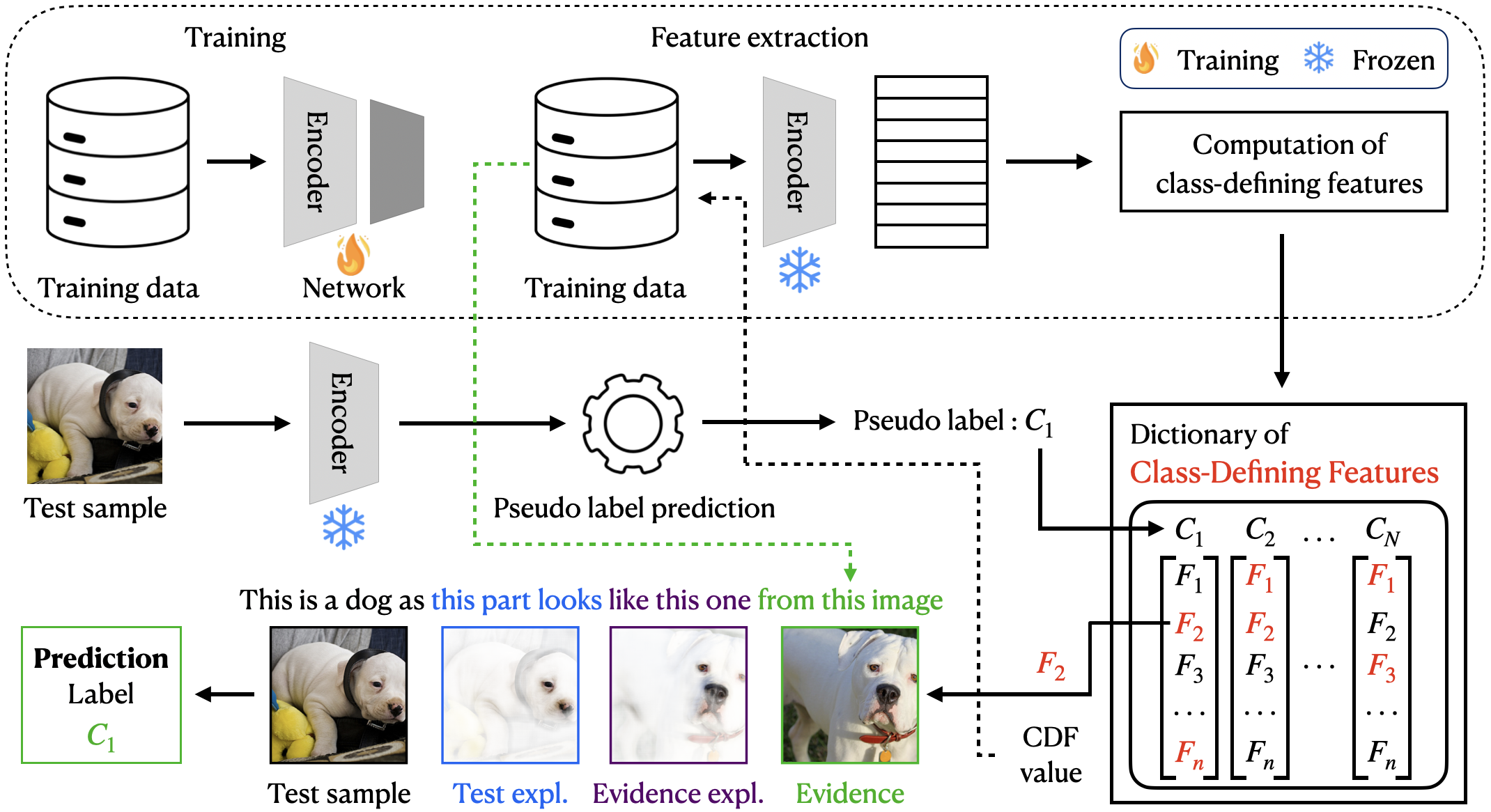}
    \caption{\methodname\ method overview. }
    \label{method_overview}
    \vspace{-0.4cm}
\end{figure}

An overview of \methodname\ is presented in Figure \ref{method_overview}. The figure demonstrates an example where a single Class Defining Feature (CDF) is used for prediction. For every test image, the final decision-making step aligns with human-interpretable reasoning: \textit{'This image is classified as a dog because this region of the image resembles the corresponding region of this training image'}. This explanation fully corresponds to the underlying computations, providing a faithful and complete representation of the decision process, i.e. not an approximation of the computation. We train a $\operatorname{B-cos}$ network, an inherently interpretable model, on the training data. Using the train features from this encoder, we compute the CDFs. During inference, we project the test image into the CDF space using a pseudo-label. For each CDF feature, we retrieve the closest matching training data point. Projecting the CDF features into image space allows us to localize the prototypical regions in the test image that correspond to the training data. The final prediction is obtained through majority voting of the labels assigned by each CDF feature.

\subsection{Preliminaries: $\operatorname{B-Cos}$ architecture}

A $\operatorname{B-cos}$ encoder generates a reliable explanation of its computation. $\operatorname{B-cos}$ networks are neural networks in which all the linear layers (along with activations) are replaced by $\operatorname{B-cos}$ layers. For more details on the formulation and training of these networks, we refer the reader to \cite{bohle2022b, bohle2024b} and to Appendix \ref{sec:Appendix_BCos_description}. Operation of a $\operatorname{B-cos}$ layer at a node for an input $\mathbf{x}$ and weights $\mathbf{w}$ leading to the node is given by
\vspace{-0.2cm}
\begin{equation}
\operatorname{B\text{-}cos}(\mathbf{x}; \mathbf{w}) = \|\mathbf{x}\| \cdot \|\mathbf{w}\| \cdot \left|\cos(\angle(\mathbf{x}, \mathbf{w}))\right|^{B}  \cdot \operatorname{sign}\left(\cos(\angle(\mathbf{x}, \mathbf{w}))\right),
\label{Bcos-equ} 
\end{equation}

where $B$ is a hyper-parameter that influences the extent to which alignment between $\mathbf{x}$ and $\mathbf{w}$ contributes to the magnitude of the output. Replacing linear layers with $\operatorname{B-cos}$ layers removes the need for other explicit non-linearity while training the network. Given an input, $\operatorname{B-cos}$ layer becomes a linear layer followed by a scalar multiplication (the cosine score: Equation \ref{Bcos-equ}). As each layer becomes a linear operation, the neural network collapses into a single linear transform that faithfully summarises the entire model computations. Moreover, the $\operatorname{B-cos}$ layers introduce alignment pressure on their weights during optimization. For the output of a node to be high, the input must align well with the node’s incoming parameters, indicated by a high value of $\cos(\angle(\mathbf{x},\mathbf{w}))$. In short, we choose the $\operatorname{B-cos}$ network for two reasons: (a) $\operatorname{B-cos}$ has an input-dependent non-linearity which collapses the encoder computations into a linear transformation for a test sample, and (b) the collapsed linear operation (i.e., matrix) is aligned to the input sample when the output is high.

Given an input image $\mathbf{x}$, $(L+1)$-layer $\operatorname{B-cos}$ network collapses into a linear layer. This matrix is aligned with the input if the output is high. The $(L+1)$-layer transformation can be presented as a shortcut representation
\begin{equation}
    W_{1\rightarrow (L+1)}(\mathbf{x};\theta) = W_{(L+1)}\circ W_{L}... \circ W_1 (\mathbf{x}; \theta),
    \label{1_L_transformation}
\end{equation}

The final output is obtained as 
\begin{equation}
    f(\mathbf{x}; \theta) =  W_{1\rightarrow (L+1)} (\mathbf{x};\theta)\mathbf{x},
\label{final_output}
\end{equation}

We modify the previous formulation of $\operatorname{B-cos}$ to get the explanation of features that are activated by the input. Previous work has also shown that $\operatorname{B-cos}$ transformers inherently learn human-interpretable features. We compute the explanation for a feature $i$ in the $L^{\text{th}}$ layer as ${W_{1\rightarrow L}(\mathbf{x}; \theta)}^i$.

\subsection{Compositional Explanations using Prototypes (\methodname)}

We present the complete methodology in Algorithm \ref{alg:one}. Hereafter $\arg_x \mathrm{top}_k [\cdot]$ denotes the generalisation of the $\arg \max_x [\cdot]$ operator where the maximum is replaced with top $k$ values. The algorithm starts (\textbf{Step 1}) with calculation of embeddings $W_{1\rightarrow L} (\mathbf{x}; \theta)$ (Encoder stage in Figure \ref{method_overview}). It proceeds with the pseudo-label prediction (\textbf{Step 2}), and the selection of the CDFs for a given pseudo-label (\textbf{Step 3}). The per-feature predictions are calculated from the CDFs for top $M$ CDFs for the pseudo-label and $K$ nearest neighbours (\textbf{Step 4}).  In \textbf{Step 5}, we calculate the corresponding explanations. It is important to see that instead of one label the method gives a number of predictions, one per every feature and per every nearest neighbour. Further in the experimental section, we calculate the aggregated prediction as a mode of the prediction set $G(\mathbf{x}; \theta)$.

\begin{algorithm}[h]
\setstretch{1.35}
\caption{Compositionally explainable classifier \methodname\ }\label{alg:one}
\vspace{0.3cm}
\KwData{ Image $\mathbf{x}$; training dataset $\mathcal{D}$; \newline  class-defining features $P^C = \{P^c\ \forall c \in \mathbb{C}\}$; a number of features $M$ to be explained; \newline a number of nearest neighbours $K$}
\KwResult{$M \times K$ per-feature predictions $G (\mathbf{x}, \theta) = \{g^i,j(\mathbf{x}; \theta)\}_{i\in [1\ldots P^C], j\in [1, K]}$;\newline explanations $E(\mathbf{x}; \mathcal{D}, P^C)$ for retrieved concepts  }
\vspace{0.3cm}

\begin{enumerate}
    \item Calculate $W_{1\rightarrow L} (\mathbf{x}; \theta)$ as per Equation (\ref{1_L_transformation})
    \item Predict the nearest-neighbour pseudo-label class using Equation (\ref{pseudo_label_prediction})
    \item {Using the pseudo-label, select the top $M$ scalar class-defining features $P^{\tilde{g} (\mathbf{x}; \theta)}$  (see Equation (\ref{pseudo_label_mutual_information})}
    \item{Calculate the per-feature predictions $G (\mathbf{x}; \theta, P^C)$ from the class-defining features $P^{\tilde{g} (\mathbf{x}; \theta)}$ according to Equations (\ref{final_label_computation}) and (\ref{d_computation})}
    \item { Calculate the explanations for the $K$ nearest neighbours for every class-defining feature according to Equation \ref{equation_explanations}}
\end{enumerate}
\end{algorithm}

We define a \textit{training dataset} $\mathcal{D} = \{\mathbf{d}_1, \mathbf{d}_2, \ldots \mathbf{d}_n\}$ which contains a set of reference image samples, annotated by the labels  $\mathcal{L} = \{\mathbf{l}_1, \mathbf{l}_2, \ldots \mathbf{l}_n\}$ from a label set $\mathbb{C}$ as $l(\mathbf{d}_i) = \mathbf{l}_{i} \ \forall i \in [1, n]$.

We focus our experiments on the final layer and analyse its properties through the lens of transformation $W_{1\rightarrow L}$, which has shape $C_{L} \times  (W\cdot H\cdot D)$. Here $C_{L}$ is the number of features in the last layer ($L^{\text{th}}$ layer). The per-feature attribution explanations for a given input $\mathbf{x}$ is given by $\textbf{s}^\cdot_{L}  (\mathbf{x}; \theta)$ defined as follows:
\begin{equation}
    W_{1\rightarrow L} (\mathbf{x}; \theta) = \left(
    \textbf{s}^1_{L}  (\mathbf{x}; \theta),
    \textbf{s}^2_{L}  (\mathbf{x}; \theta),
    \ldots ,
    \textbf{s}^{C_{L}}_{L}  (\mathbf{x}; \theta)\\
    \right)^T,
\end{equation}

\vspace{-0.3cm}
\paragraph{ Step 2} uses the following equation to compute the pseudo-label class:
\begin{equation}
\tilde{g} (\mathbf{x}; \theta) = l(\arg \min_{\mathbf{d}} \{\ell^2 (W_{1\rightarrow L} (\mathbf{x}; \theta) \mathbf{x}, W_{1\rightarrow L}(\mathbf{d}; \theta)\mathbf{d})\ \forall \mathbf{d} \in \mathcal{D}\}) 
\label{pseudo_label_prediction}
\end{equation}

\vspace{-0.3cm}
\paragraph{In Step 3,} for the dataset $\mathcal{D}$, we calculate the top $M$ scalar \textit{class-defining features} {$P^{c}$} for class $c\in \mathbb{C}$ by using maximum mutual information:
\begin{equation}
    F = \{W_{1\rightarrow L} (\mathbf{d}, \theta) \mathbf{d} \ \ \forall \mathbf{d} \in \mathcal{D}\},
    F_j = \{\textbf{s}^j_{L} (\mathbf{d}, \theta) \mathbf{d} \ \ \forall \mathbf{d} \in \mathcal{D}\},
\end{equation}
\begin{equation}
    P^{c} =\{\arg_j \mathrm{top}_M  I (F_j, l (F_j) = c)\ \forall j \in [1\ldots C_{L}]\},
    \label{pseudo_label_mutual_information}
\end{equation}
where $c\in \mathbb{C}$ is a label for class $c$, $l(F_j)$ is a ground-truth label operator for the feature $F_j$, and the mutual information $I (X, Y)$ is defined as 
\begin{equation}
    I(X, Y) = \sum_{\left<x, y\right> \in \left<X, Y\right>} p(x, y) \log \left(\frac{p(x, y)}{p(x) p(y)}\right).
\end{equation}

The introduction of pseudo-labels is necessary for the selection of a small number of CDFs and therefore restricting the explanation to a small number of features. They constitute the initialisation for the decision-making process, which allows bootstrapping the selection of class-defining features. 

\vspace{-0.3cm}
\paragraph{Step 4} calculates the per-feature predictions $G (\mathbf{x}; \theta, P^C)$ through the following equations: 
\begin{equation} G (\mathbf{x}; \theta, P^C) = l(\mathbf{D}^*(\mathbf{x}, \theta, P^C)),
\label{final_label_computation}
\end{equation}
\begin{equation}
\mathbf{D}^* (\mathbf{x}, \theta, P^C) = \{\arg_{\mathbf{d}} {\mathrm{top}_K} \{-\ell^2 ([W_{1\rightarrow L} (\mathbf{x}; \theta) \mathbf{x}]_{f}, [W_{1\rightarrow L}(\mathbf{d}; \theta)\mathbf{d}]_{f})\ \forall \mathbf{d} \in \mathcal{D}\}\}_{f \in {P^{\tilde{g} (\mathbf{x}; \theta)}}}
\label{d_computation}
\end{equation}

\vspace{-0.3cm}
\paragraph{In Step 5,}  explanations for the CDF are calculated using the following equations: 
\begin{equation}
E(\mathbf{x}; \mathcal{D}, P^{C}) = E(\mathbf{x}; \mathcal{D}, P^{\tilde{g} (\mathbf{x}; \theta)}) = \{\left<s^i_{L} (\mathbf{x}_i, \theta), s^i_{L} (\mathbf{d}_i^{\mathrm{nearest}}, \theta)\right>\ \forall i \in P^{\tilde{g} (\mathbf{x}; \theta)}\},
\label{equation_explanations}
\end{equation} 
where the training samples' features, nearest to a class-defining feature of the testing image $\mathbf{x}$, are calculated as $ \mathbf{d}_i^{\mathrm{nearest}} = \arg_\mathbf{d} \textrm{top}_K \{\ell^2 ((W_{1\rightarrow L} (\mathbf{d};\theta)\mathbf{d})_i, (W_{1\rightarrow L} (\mathbf{x};\theta)\mathbf{x})_i),  \forall \mathbf{d} \in \mathbf{D}^*\}$ and $s^i_{L} (\mathbf{d}, \theta)$ is $i$-th row of $W_{1\rightarrow L} (\mathbf{d}, \theta)\}$.

\subsection{Demonstration of meeting the desiderata}
\label{desiderata_demonstration}
We define the criterion of sufficiency of the explanation and demonstrate how and in which conditions we meet this criterion. In Table \ref{sparsity_evaluation} the experimental section, we also outline how \methodname\ addresses the requirements of \textbf{sparsity}. We address the question of \textbf{fidelity} experimentally, by measuring insertion and deletion metrics in Section \ref{eval_interpretability}.

We address \textbf{necessity} (i.e., presence of the elements of the explanation necessary for the decision making) of the explanations $E(\mathbf{x}, \mathcal{D}, P^C)$ from Equation \ref{equation_explanations} by visualising the elements of exact same nearest-neighbour samples that are present in the decision-making procedure in Equation \ref{d_computation}.

We define \textbf{sufficiency} of the explanations $E(\mathbf{x}; \mathcal{D}, P^{C})$ in a way that the same explanation would imply the same output:
\begin{equation}
    \forall \mathbf{x}, \mathbf{x}' E(\mathbf{x}'; \mathcal{D}, P^{C})= E(\mathbf{x}; \mathcal{D}, P^{C})  \implies G(\mathbf{x}'; \theta, \mathcal{D}, P^{C}) = G(\mathbf{x}; \theta, \mathcal{D}, P^{C})
    \label{sufficiency}
\end{equation}

\begin{theorem}{Assume $\tilde{g}(\mathbf{x}; \theta)= {g}(\mathbf{x}; \theta)\ \forall {g}(\mathbf{x}; \theta) \in {G}(\mathbf{x}; \theta)$. Then the explanation $E(\mathbf{x}; \mathcal{D})$ is sufficient for the prediction $G(\mathbf{x}; \theta, \mathcal{D})$} according to Algorithm \ref{alg:one}.
\label{theorem_faithfulness}
\end{theorem}
\begin{proof} 

 Suppose that $E(\mathbf{x}'; \mathcal{D}, P^{C})= E(\mathbf{x}; \mathcal{D}, P^{C})$ and $G(\mathbf{x}'; \theta, \mathcal{D}, P^{C}) \ne G(\mathbf{x}; \theta, \mathcal{D}, P^{C})$ for some $\mathbf{x}, \mathbf{x}'$. Using the assumption that $\tilde{g}(\mathbf{x}; \theta)= {g}(\mathbf{x}; \theta)\ \forall {g}(\mathbf{x}; \theta) \in {G}(\mathbf{x}; \theta)$, one can note that the two sets $\mathbf{D}^* (\mathbf{x}, \theta, P^C), \mathbf{D}^* (\mathbf{x}', \theta, P^C)$ cannot possibly be the same as the labels of the two sets are different and the same training datum $\mathbf{d}$ cannot have two different labels, i.e. $G(\mathbf{x}'; \theta, \mathcal{D}, P^{C}) \ne G(\mathbf{x}; \theta, \mathcal{D}, P^{C})$ means that $\mathbf{D}^* (\mathbf{x}, \theta, P^C) \ne \mathbf{D}^* (\mathbf{x}', \theta, P^C)$. This means that the explanations $E(\mathbf{x}, \mathcal{D}, P^C)$ and $E(\mathbf{x}', \mathcal{D}, P^C)$ are calculated in Equation \ref{equation_explanations} over two different subsets of training samples and therefore cannot possibly be the same. Therefore, we can see that, by contradiction, Equation \ref{sufficiency} holds true for Algorithm \ref{alg:one}.

\end{proof}

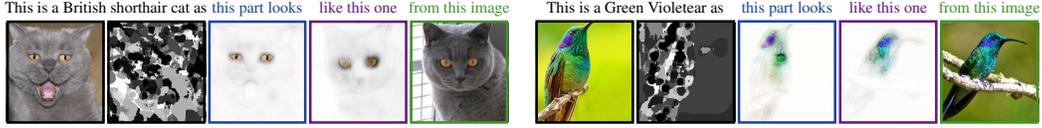
\begin{figure}
\resizebox{0.5\textwidth}{!}{
\begin{tikzpicture}[domain=0:15]
\node at (-4, -2.5) {\includegraphics[height=1.95cm]{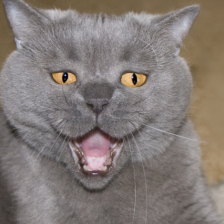}};
\node at (-2, -2.5) {\includegraphics[height=1.95cm]{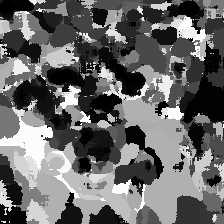}};
\node at (0, -2.5) {\includegraphics[height=1.95cm]{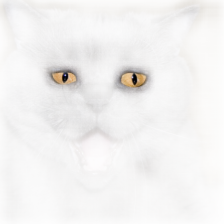}};
\node at (2, -2.5) {\includegraphics[height=1.95cm]{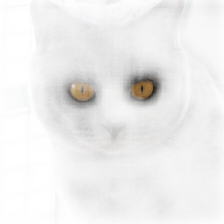}};
\node at (4, -2.5) {\includegraphics[height=1.95cm]{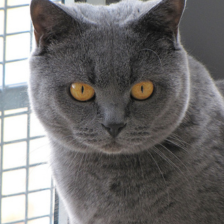}};

\draw [draw=black, line width=0.5mm] (-4.95,-1.5) rectangle (-3.05,-3.5);
\draw [draw=black, line width=0.5mm] (-2.95,-1.5) rectangle (-1.05,-3.5);
\draw [draw={rgb:red,40;green,101;blue,238}, line width=0.5mm] (-0.95,-1.5) rectangle (0.95,-3.5);
\draw [draw={rgb:red,86;green,20;blue,107}, line width=0.5mm] (1.05,-1.5) rectangle (2.95,-3.5);
\draw [draw={rgb:red,81;green,181;blue,52}, line width=0.5mm] (3.05,-1.5) rectangle (4.95,-3.5);

\node at (-3,-1.25) {\makecell{\footnotesize{This is a British shorthair cat as}}};
\node[text={rgb:red,40;green,101;blue,238}] at (0,-1.25) {\makecell{\footnotesize{this part looks}}};
\node[text={rgb:red,86;green,20;blue,107}] at (2,-1.25) {\makecell{\footnotesize{like this one}}};
\node[text={rgb:red,81;green,181;blue,52}] at (4,-1.25) {\makecell{\footnotesize{from this image}}};

\end{tikzpicture}
}
\resizebox{0.5\textwidth}{!}{
\begin{tikzpicture}[domain=0:15]
\node at (-4, -2.5) {\includegraphics[height=1.95cm]{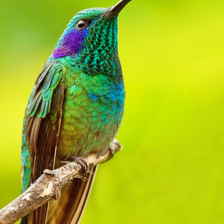}};
\node at (-2, -2.5) {\includegraphics[height=1.95cm]{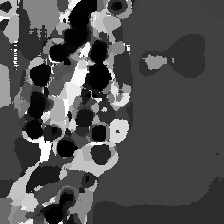}};
\node at (0, -2.5) {\includegraphics[height=1.95cm]{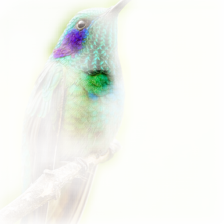}};
\node at (2, -2.5) {\includegraphics[height=1.95cm]{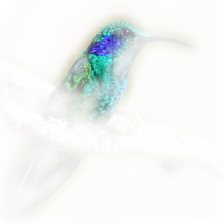}};
\node at (4, -2.5) {\includegraphics[height=1.95cm]{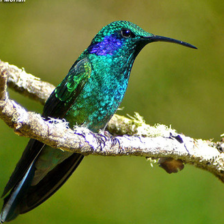}};

\draw [draw=black, line width=0.5mm] (-4.95,-1.5) rectangle (-3.05,-3.5);
\draw [draw=black, line width=0.5mm] (-2.95,-1.5) rectangle (-1.05,-3.5);
\draw [draw={rgb:red,40;green,101;blue,238}, line width=0.5mm] (-0.95,-1.5) rectangle (0.95,-3.5);
\draw [draw={rgb:red,86;green,20;blue,107}, line width=0.5mm] (1.05,-1.5) rectangle (2.95,-3.5);
\draw [draw={rgb:red,81;green,181;blue,52}, line width=0.5mm] (3.05,-1.5) rectangle (4.95,-3.5);

\node at (-3,-1.25) {\makecell{\footnotesize{This is a Green Violetear as}}};
\node[text={rgb:red,40;green,101;blue,238}] at (0,-1.25) {\makecell{\footnotesize{this part looks}}};
\node[text={rgb:red,86;green,20;blue,107}] at (2,-1.25) {\makecell{\footnotesize{like this one}}};
\node[text={rgb:red,81;green,181;blue,52}] at (4,-1.25) {\makecell{\footnotesize{from this image}}};

\end{tikzpicture}
}
\caption{Examples of \textit{\methodname\ panel} interpretations for Oxford-IIIT Pets (left) and CUB-200-211 dataset (right).}
\vspace{-0.4cm}
    \label{Examples_of_interpretations}
\end{figure}

\begin{table}
\centering
 \caption{Evaluation of performance against ProtoPNet (\cite{chen2019looks}), B-Cos (\cite{bohle2024b})  and common deep-learning baselines on CUB-200-2011 (full images), the values denoted by $^*$ are obtained from \cite{donnelly2022deformable}}
 \begin{tabular}{lccccc} 
 \toprule
\textbf{Architecture} & \textbf{Baseline} & \textbf{ProtoPNet} & \textbf{B-cos} & \textbf{\methodname} \\ [0.5ex] 
 \midrule
 ResNet34 (\cite{he2016deep}) & $76.0*$ & $72.4^*$ & $74.3$ & $73.8$ \\ 
 ResNet152 (\cite{he2016deep}) & $79.2^*$ & $74.3^*$ & $76.5$ &  $76.2$ \\
 DenseNet121 (\cite{huang2017densely}) & $78.2*$ & $74.0^*$  & $73.6$ &  $73.2$ \\
 DenseNet161 (\cite{huang2017densely}) & $80.0^*$ & $75.4^*$ & $76.1$ &  $76.1$ \\
 \bottomrule
 \end{tabular}
 \label{accuracy_evaluation}
\end{table}
\begin{table}
\centering
 \caption{Evaluation of performance against B-Cos  (\cite{bohle2024b}) and baseline ViT (\cite{dosovitskiy2021an}), $K$-NN refers to the baseline of B-cos + $K=3$ nearest neighbours, pretrained on ImageNet}
\begin{tabular}{lccccc} 
 \toprule
\textbf{Dataset} & \textbf{ViT} & \textbf{B-cos} & \textbf{$k$-NN} & \textbf{\methodname} \\ [0.5ex] 
\midrule
 Oxford-IIIT Pets & $90.32 \pm 0.03$ & $89.32 \pm 0.13$ & $89.23\pm0.11$ & $87.73\pm0.21$ \\ 
 CUB-200-2011 & $79.62 \pm 0.04$ & $79.23 \pm 0.08$ & $78.98\pm 0.06$  & $74.14 \pm0.18 $\\
 Stanford Cars & $90.72\pm0.32$ & $86.53 \pm 0.31$ & $87.95  \pm 0.24$  & $86.81 \pm 0.24$ \\
 CIFAR-10 & $93.34\pm0.08$ & $93.10\pm 0.15$ & $93.28\pm 0.09$  & $91.21 \pm 0.19$\\
 CIFAR-100 & $78.61\pm 0.03$ & $76.07 \pm 0.06$ & $74.23 \pm 0.04$ & $76.42 \pm0.12$\\
ImageNet & $78.90 \pm 0.24$ & $77.78 \pm 0.24$ & $75.16 \pm 0.24$ & $74.28 \pm 0.38$ \\
 \bottomrule
 \end{tabular}
 \label{accuracy_eval_datasets}
\end{table}

\section{Experiments and discussion}

In this section, we evaluate \methodname\ through a series of quantitative and qualitative experiments. We assess the model's performance on standard benchmarks (accuracy, fidelity, and sparsity) to validate the method's effectiveness. We compare the accuracy of \methodname\ with other baseline methods. We also show the robustness of the model performance across different backbones. We demonstrate the fidelity of \methodname\ by evaluating the method using causal matrices (\cite{ghorbani2019towards}). Additionally, we present qualitative analyses by visualizing the prototypical regions identified during inference, providing insights into the interpretability and decision-making process of the model. These experiments highlight the model’s ability to provide transparent and faithful explanations while maintaining competitive accuracy.

Figure \ref{Examples_of_interpretations} shows the explanation generated by the method using only one evidence sample and one feature alone used for prediction. The second image in the panel shows the super-pixel like segmentation generated based on the dominant CDF feature for every pixel ($\arg \max_i \{\mathbf{s}_L^i(\mathbf{x}, \theta)\} \forall i \in (\mathbf{s}_L^1(\mathbf{x}, \theta), \cdots, \mathbf{s}_L^i(\mathbf{x}, \theta)), \cdots \mathbf{s}_L^{C_L}(\mathbf{x}, \theta)$). We present more interpretation examples in Appendix \ref{additional_qualitative_results}.

\subsection{Datasets}

\paragraph{Datasets} We train and evaluate the presented model on a number of commonly-used computer vision datasets. CIFAR-10\&100 \citep{krizhevsky2009learning} contain  generic natural images from 10 and 100 diverse classes respectively. CUB-200-2011 \citep{welinder2010caltech} is a commonly used dataset for evaluating interpretable vision models, which contains $200$ fine-grained classes of birds. Stanford cars dataset \citep{krause20133d} contains $196$ classes of cars. Oxford-IIIT Pets \citep{parkhi2012cats} contains a fine-grained collection of images of $37$ classes of cats and dogs.  Finally, we present the results on ImageNet (ILSVRC 2012) \citep{ILSVRC15} which has $1000$ diverse classes.

\vspace{-0.4cm}
\paragraph {Baselines} We compare \methodname\ to the following well-known baselines: (1) Standard architectures such as ResNet \cite{he2016deep}, DenseNet \cite{huang2017densely} and ViT \cite{dosovitskiy2021an} (2) the $\operatorname{B-Cos}$ counterparts of the aforementioned architectures \citep{bohle2022b,bohle2024b} and (3) a number of interpetable and explainable ML methods including ProtoPNet \citep{chen2019looks}, Deformable ProtoPNet \citep{donnelly2022deformable} and CAM \citep{zhou2016learning}.

\vspace{-0.4cm}
\paragraph {Experimental setup} We pretrain the backbone $\operatorname{B-cos}$ models on ImageNet \citep{ILSVRC15} and then on target datasets. The details of the experimental setup, hardware configuration and hyperparameters are described in Appendix \ref{experimental_setup_section}.

\vspace{-0.3cm}
\paragraph{Accuracy evaluation}

\begin{table}[t]
\centering
\small
 \caption{Interpretability vs accuracy, \% (adopted from \cite{donnelly2022deformable}, all values except from ViT, $\operatorname{B-cos}$, and \methodname, come from there)}
 \begin{tabular}{p{0.20\textwidth}p{0.25\textwidth}p{0.08\textwidth}p{0.25\textwidth}p{0.08\textwidth}}  
 \toprule
\textbf{Interpretability} 
    & \makecell{\textbf{Method}} & \textbf{ CUBS} 
    & \makecell{\textbf{Method}} & \textbf{ Cars}  
    \\\midrule

 \makecell{None} 
    & \makecell{ViT \\ \small{(\cite{dosovitskiy2021an})}} 
    & \makecell{$79.6$} 
    & \makecell{ViT \\ \small{(\cite{dosovitskiy2021an})}} 
    & \makecell{$92.72$}  
    \\ \midrule
 \makecell{Part-level\\ attention} 
    & \makecell{TASN \\ \small{(\cite{zheng2019looking})}} 
    & \makecell{$87.0$} 
    & \makecell{FCAN \\ \small{(\cite{liu2016fully})} \\ RA-CNN \\ \small{(\cite{fu2017look})}} 
    & \makecell{$84.2$ \\ $87.3$}  
    \\ \midrule 
 \makecell{Part-level\\ attention \\ + Prototypes} 
    & \makecell{ProtoPNet \\ \small{(\cite{chen2019looks})} \\ Def. ProtoPNet \\ \small{(\cite{donnelly2022deformable})}} 
    & \makecell{$81.1$ \\ $86.4$} 
    & \makecell{ProtoPNet \\ \small{(\cite{chen2019looks})} \\ Def. ProtoPNet \\ \small{(\cite{donnelly2022deformable})}} 
    & \makecell{$77.3$ \\ $86.5$}  
    \\ \midrule
 \makecell{Object-level\\ attention} 
    & \makecell{CAM \\ \small{(\cite{zhou2016learning})} \\ CSG \\ \small{(\cite{liang2020training})} \\ B-cos \\ \small{(\cite{bohle2024b})}} 
    & \makecell{$63.0$ \\ $78.5$ \\ $79.2$} 
    & \makecell{B-cos \\ \small{(\cite{bohle2024b})}} 
    & \makecell{$86.5$}  
    \\ \midrule 
 \makecell{Object-level\\ attention\\ + Prototypes} 
    & \makecell{\methodname} 
    & \makecell{$74.0$} 
    & \makecell{\methodname} 
    & \makecell{$86.80$}  
    \\ \bottomrule
 \end{tabular}
 \label{accuracy_eval_categories}
\end{table}

Table  \ref{accuracy_evaluation} shows that \methodname\ provides competitive accuracy compared to the baselines of $\operatorname{B-Cos}$/ViT on the full-frame CUB-200-2011 dataset \citep{welinder2010caltech}. In all cases, the accuracy is calculated as the mode of the values of output predictions $G (x, \theta)$. In this experiment, the B-cos architecture for both B-cos and \methodname\ mirrors the corresponding deep-learning baseline, i.e. $\operatorname{B-Cos}$/ViT is a $\operatorname{B-cos}$ counterpart of the ViT model as described in \cite{bohle2024b}. Table  \ref{accuracy_evaluation}  does not list the ViT results. This is due to the reason that, as suggested in \cite{xue2022protopformer}, ProtoPNet cannot be used with the ViT architecture without substantial modifications.

In Table \ref{accuracy_eval_datasets}, we evaluate the performance of the model on a variety of datasets using just the $\operatorname{B-cos}$/ViT backbone as outlined in Appendix \ref{experimental_setup_section}. Additionally, in Appendix \ref{performance_nf}, we show that surprisingly, \methodname\ provides impressive capabilities for learning without finetuning on the target data, opening up the potential for adaptation of the method to the new datasets without finetuning. In this setting, the models were pretrained on ImageNet, and then, during the inference stage, matched by the nearest neighbours procedure as described in Algorithm \ref{alg:one}. The $k$-NN results were calculated for $k=3$ using the B-cos backbone.

In Table \ref{accuracy_eval_categories}, we demonstrate the trade-off between interpretability and accuracy across the categories. Following \cite{chen2019looks}, we compare different types of explainable and interpretable models. Many of the \textit{post hoc} attribution-based methods, as well as original $\operatorname{B-cos}$ model, provide interpretation through object-level attention, while patch-based methods such as, notably, ProtoPNet (\cite{chen2019looks}) are referred to as part-level attention methods.

\subsection{Evaluation  of the Interpretability Properties}
\label{eval_interpretability}
Table \ref{insertion_deletion} presents the evaluation of Average Drop, Average Increase, C-insertion, and C-deletion metrics (\cite{ghorbani2019towards}), which are the standard metrics to quantify fidelity of the explanations both for \textit{ante hoc} (by-design) and \textit{post hoc} methods. The average drop (increase) metrics are calculated as the drop (increase) in performance of prediction when we drop (add) $50\%$ of the pixels with the lowest (highest) attribution.  The C-insertion and C-deletion metrics represent the area under the curve for insertion (deletion) of pixels in the increasing (decreasing) order of pixel attribution value.  

\begin{table}
\centering
 \caption{Evaluation of Average Drop, $\%$, Average Increase, $\%$, C-insertion and C-deletion metrics (results marked with $^*$ are taken from \cite{zeng2023abs})}
 \begin{tabular}{lcccc} 
 \toprule
\textbf{Method} & \textbf{Drop $\downarrow$} & \textbf{Increase  $\uparrow$} & \textbf{C-insertion} $\uparrow$ & \textbf{C-deletion} $\downarrow$\\ [0.5ex] 
 \midrule
 Grad-CAM (\cite{selvaraju2017grad}) & $41.5^*$ & $20.8^*$ & $0.4626^*$ & $0.1110^*$\\ 
 Grad-CAM++ (\cite{chattopadhay2018grad}) & $40.8^*$ & $22.3^*$ & $0.4484^*$ & $0.1179^*$\\ 
 SGCAM++ (\cite{omeiza2019smooth}) & $41.1^*$ & $23.4^*$ & $0.4504^*$ & $0.1169^*$\\ 
 Score-CAM (\cite{wang2020score}) & $35.6^*$ & $29.5^*$ & $0.4929^*$ & $0.1099^*$ \\ 
 Group-CAM (\cite{zhang2021group}) & $35.7^*$ & $29.7^*$ & $0.4930^*$ & $0.1108^*$ \\ 
Abs-CAM (\cite{zeng2023abs}) & $34.2^*$ & $30.1^*$ & $0.4949^*$ & $\mathbf{0.1096}^*$ \\ 

\methodname  & $41.3$ & $\mathbf{36.5}$ & {$\mathbf{0.7365}$} &  $0.1214$ \\ 

 \bottomrule
 \end{tabular}
 \label{insertion_deletion}
\end{table}

\begin{table}
\centering
 \caption{Evaluation of sparsity using PQ-index (larger is better) between \methodname\ and ViT (\cite{dosovitskiy2021an}) }

 \begin{tabular}{cccccc} 
 \toprule
\textbf{\makecell{Method/\\Dataset}} &  \textbf{\makecell{Oxford-IIIT \\ Pets}} & \textbf{\makecell{CUB-200-2011 \\ }} & \textbf{\makecell{Stanford Cars \\ }} & \textbf{\makecell{CIFAR-10}} & \textbf{\makecell{CIFAR-100}} \\ [0.5ex] 
 \midrule
 ViT & 0.21  &  0.42 & 0.61  & 0.63 & 0.45\\ 
\methodname & \textbf{0.26} & \textbf{0.52} & \textbf{0.67} & \textbf{0.65} & \textbf{0.48}\\ 
 \bottomrule
 \end{tabular}
 \label{sparsity_evaluation}
\end{table}
The property of sparsity is crucial for selecting meaningful class-defining features. Low sparsity would mean that more CDF need to be selected to meaningfully represent the class. To measure sparsity, we use the PQ-Index sparsity measure (\cite {diao2022pruning}). For the vector $\mathbf{w} \in \mathbb{R}^d$ it is defined as
$I_{p, q} (\mathbf{w}) = 1 - d^{q^{-1} -p^{-1}}\frac{|\mathbf{w}|_p}{|\mathbf{w}|_q},$
where $|\mathbf{w}|_q > 0$ is a $\ell^q$ norm, $0<p\le1<q$ are hyperparameters. In our analysis, we use the values $p=1, q=2$.

\begin{wrapfigure}[11]{R}{0.29\textwidth}
\vspace{-0.6cm}
    \includegraphics[width=\linewidth]{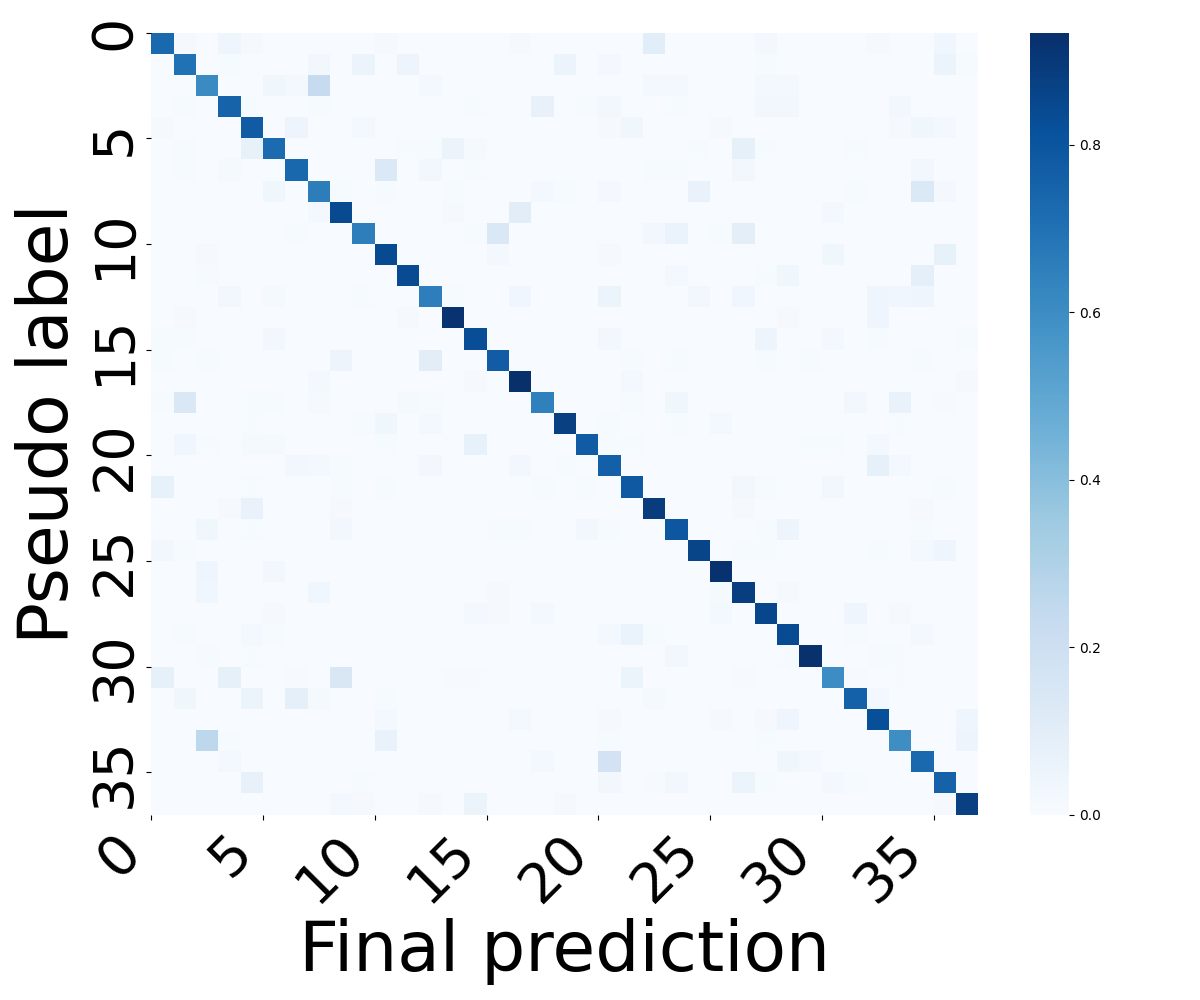}
    \centering
    \caption{ Final prediction vs pseudo-label confusion matrix on Oxford-IIIT Pets dataset}
    \label{pseudolabels_vs_real_labels}
\end{wrapfigure}

In Table \ref{sparsity_evaluation}, we measure the backbone $\operatorname{B-cos}$/ViT architecture sparsity for the extracted features $W_{1\rightarrow L} (\mathbf{x}, \theta)\mathbf{x}$ against its standard ViT counterpart, trained on the target data. We found that \methodname, based on $\operatorname{B-cos}$/ViT features, benefits from better feature sparsity compared to the ViT baseline, which justifies the use of class-defining features.
\vspace{-0.4cm}
\paragraph{Analysis of the impact of pseudo-labels}
\label{pseudolabel_analysis}

 In Figure \ref{pseudolabels_vs_real_labels}, we present the confusion matrices for the Oxford-IIIT Pets dataset, which compares the final prediction vs pseudo-labels. The extended version of this figure which compares between the true labels, the final predictions, and the pseudo-labels is included in Appendix \ref{confusion_matrices_for_pseudolabels_appendix}.

\subsection{Demonstration of prototypical explanation}
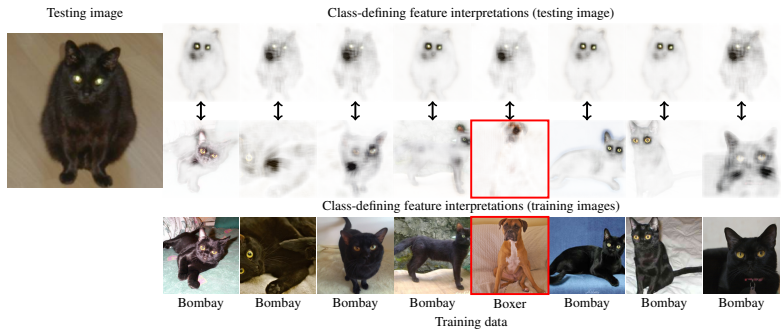
\begin{figure}
\centering
\resizebox{0.75\textwidth}{!}{
\begin{tikzpicture}[domain=0:15]
\node at (-13,-1.25) {Testing image};
\node at (-13, -3.75) {\includegraphics[height=4cm]{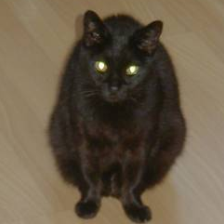}};

\node at (-3,-1.25) {Class-defining feature interpretations (testing image)};
\node at (-10, -2.5) {\includegraphics[height=2cm]{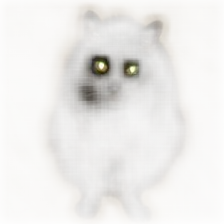}};
\node at (-8, -2.5) {\includegraphics[height=2cm]{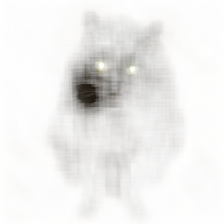}};
\node at (-6, -2.5) {\includegraphics[height=2cm]{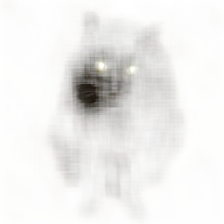}};
\node at (-4, -2.5) {\includegraphics[height=2cm]{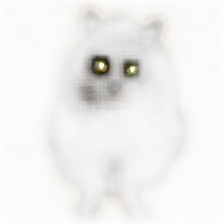}};
\node at (-2, -2.5) {\includegraphics[height=2cm]{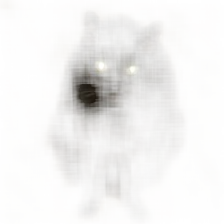}};
\node at (0, -2.5) {\includegraphics[height=2cm]{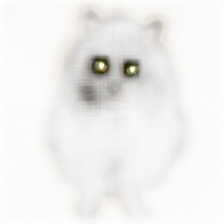}};
\node at (2, -2.5) {\includegraphics[height=2cm]{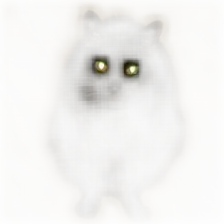}};
\node at (4, -2.5) {\includegraphics[height=2cm]{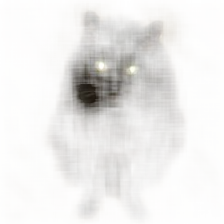}};
\draw[<->, line width=0.5mm] (-10,-3.5) -- (-10,-4);
\draw[<->, line width=0.5mm] (-8,-3.5) -- (-8,-4);
\draw[<->, line width=0.5mm] (-6,-3.5) -- (-6,-4);
\draw[<->, line width=0.5mm] (-4,-3.5) -- (-4,-4);
\draw[<->, line width=0.5mm] (-2,-3.5) -- (-2,-4);
\draw[<->, line width=0.5mm] (0,-3.5)  -- (0,-4);
\draw[<->, line width=0.5mm] (2,-3.5)  -- (2,-4);
\draw[<->, line width=0.5mm] (4,-3.5)  -- (4,-4);

\node at (-3,-6.25) {Class-defining feature interpretations (training images)};
\node at (-10, -5){\includegraphics[height=2cm]{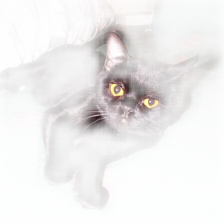}};
\node at (-8, -5){\includegraphics[height=2cm]{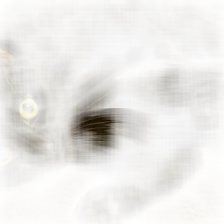}};
\node at (-6, -5){\includegraphics[height=2cm]{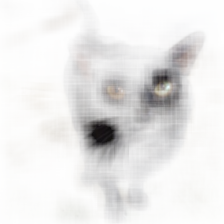}};
\node at (-4, -5){\includegraphics[height=2cm]{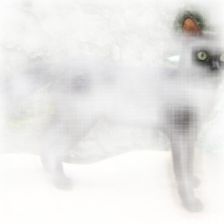}};
\node at (-2, -5){\includegraphics[height=2cm]{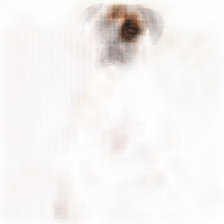}};
\node at (0, -5){\includegraphics[height=2cm]{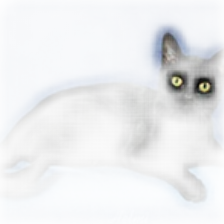}};
\node at (2, -5){\includegraphics[height=2cm]{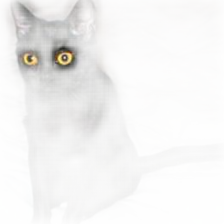}};
\node at (4, -5){\includegraphics[height=2cm]{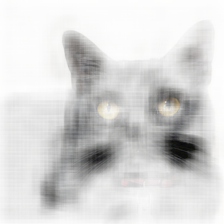}};

\node at (-10,-8.75) {Bombay};
\node at (-8,-8.75) {Bombay};
\node at (-6,-8.75) {Bombay};
\node at (-4,-8.75) {Bombay};
\node at (-2,-8.75) {Boxer};
\node at (0,-8.75) {Bombay};
\node at (2,-8.75) {Bombay};
\node at (4,-8.75) {Bombay};

\node at (-3,-9.25) {Training data};
\node at (-10, -7.5) {\includegraphics[width=1.95cm]{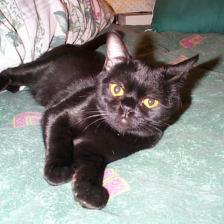}};
\node at (-8, -7.5) {\includegraphics[width=1.95cm]{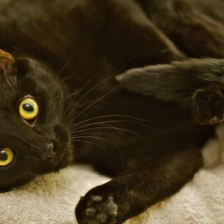}};
\node at (-6, -7.5) {\includegraphics[width=1.95cm]{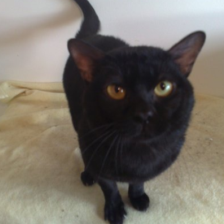}};
\node at (-4, -7.5) {\includegraphics[width=1.95cm]{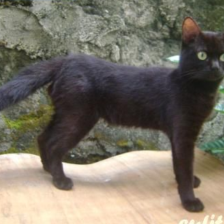}};
\node at (-2, -7.5) {\includegraphics[width=1.95cm]{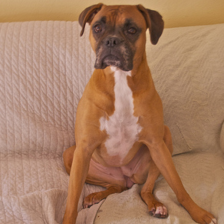}};
\node at (0, -7.5) {\includegraphics[width=1.95cm]{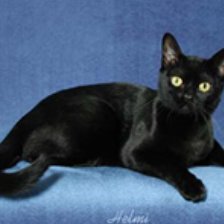}};
\node at (2, -7.5) {\includegraphics[width=1.95cm]{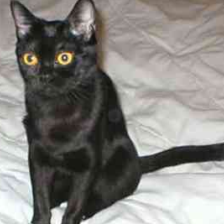}};
\node at (4, -7.5) {\includegraphics[width=1.95cm]{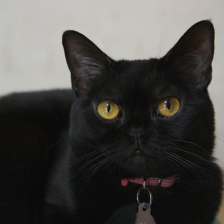}};

\draw [draw=red, line width=0.5mm] (-3,-6.5) rectangle (-1,-8.5);
\draw [draw=red, line width=0.5mm] (-3,-4) rectangle (-1,-6);

\end{tikzpicture}
}

\caption{Interpretation for a sample image from the Oxford-IIIT Pets dataset: the model correctly classifies the input image as 'Bombay cat'. This visualization demonstrates the similarity between the test image and seven training images of the 'Bombay cat' class and one image of a boxer dog (highlighted in red), offering insight into the model's decision-making process.}
    \label{qualitative_results}
\end{figure}

We qualitatively show the examples for the prototypical explanation in Figure \ref{qualitative_results}. In this figure as well as in Appendix \ref{additional_qualitative_results}, we demonstrate that on a number of use cases, the model can present both factual and counterfactual interpretations for a number of complex scenarios. The reader can note the correspondence between the features of the training and testing image attribution maps arising from explanations by class-defining features.

\subsection{Parameter sensitivity analysis}
In Figure \ref{fig:sensitivity_analysis} we compare the performance of the method depending upon the number $K$ of nearest neighbours as well as the number of features $M$ used for the prediction. This analysis shows that there is an inherent trade-off between the performance and the conciseness of explanation. 
 In Appendix \ref{Appendix_hyperparameter_choice} we also show similar experimental results in a setting similar to the state-of-the-art work. In this setting, instead of performing the predictions feature-by-feature as in Algorithm \ref{alg:one}, Step 4, we show the prediction performance using $\ell^2$ distances over the whole set of CDFs as common in the current literature.  This creates another trade-off: while it may increase the performance, the downside of such an alternative approach would be that the explanations, and the decisions, would not follow directly from the given features but from their combination.

\begin{figure}
    \centering
    \begin{subfigure}[b]{0.19\textwidth}
        \centering
        \includegraphics[width=\textwidth]{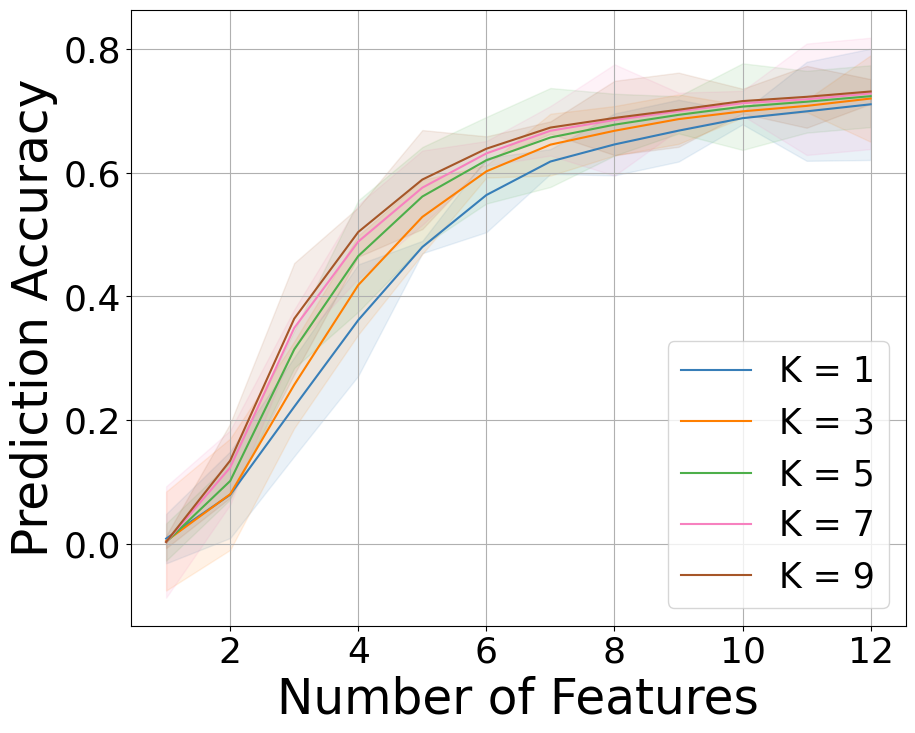}
        \caption{CUBS}
        \label{fig:om_cubs}
    \end{subfigure}
    \hfill
    \begin{subfigure}[b]{0.19\textwidth}
        \centering
        \includegraphics[width=\textwidth]{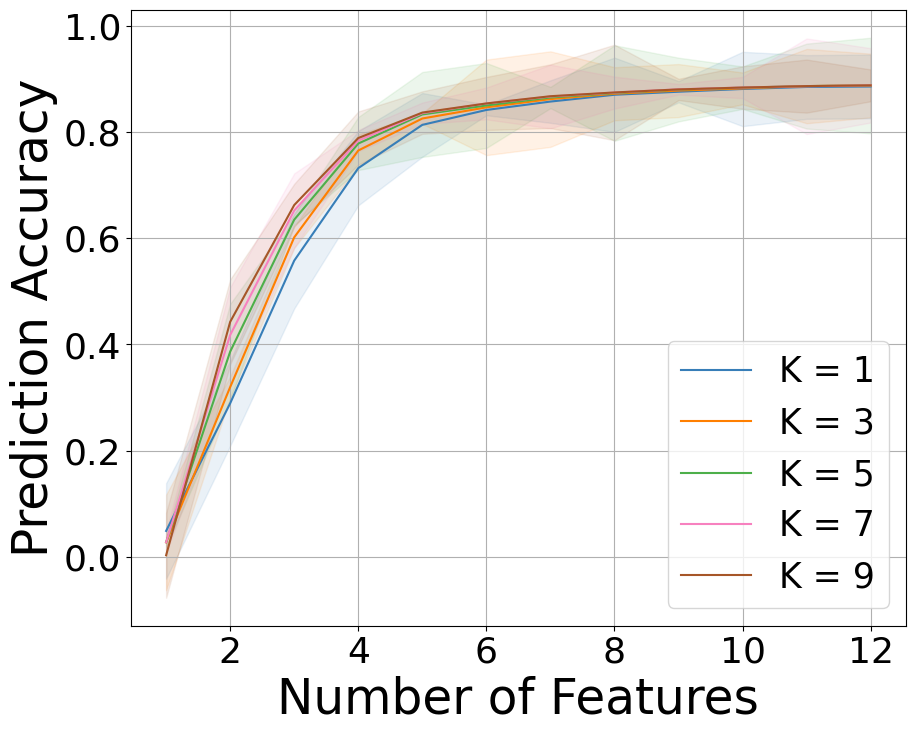}
        \caption{Oxford-IIIT Pets}
        \label{fig:om_pets}
    \end{subfigure}
    \hfill
    \begin{subfigure}[b]{0.19\textwidth}
        \centering
        \includegraphics[width=\textwidth]{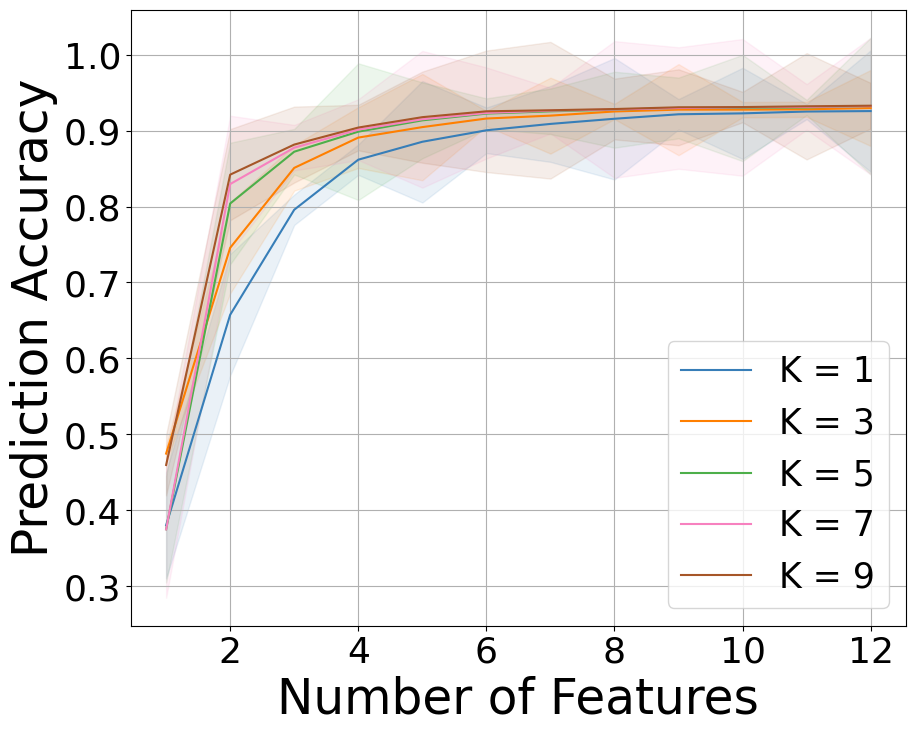}
        \caption{CIFAR-10}
        \label{fig:om_cifar10}
    \end{subfigure}
    \hfill
    \begin{subfigure}[b]{0.19\textwidth}
        \centering
        \includegraphics[width=\textwidth]{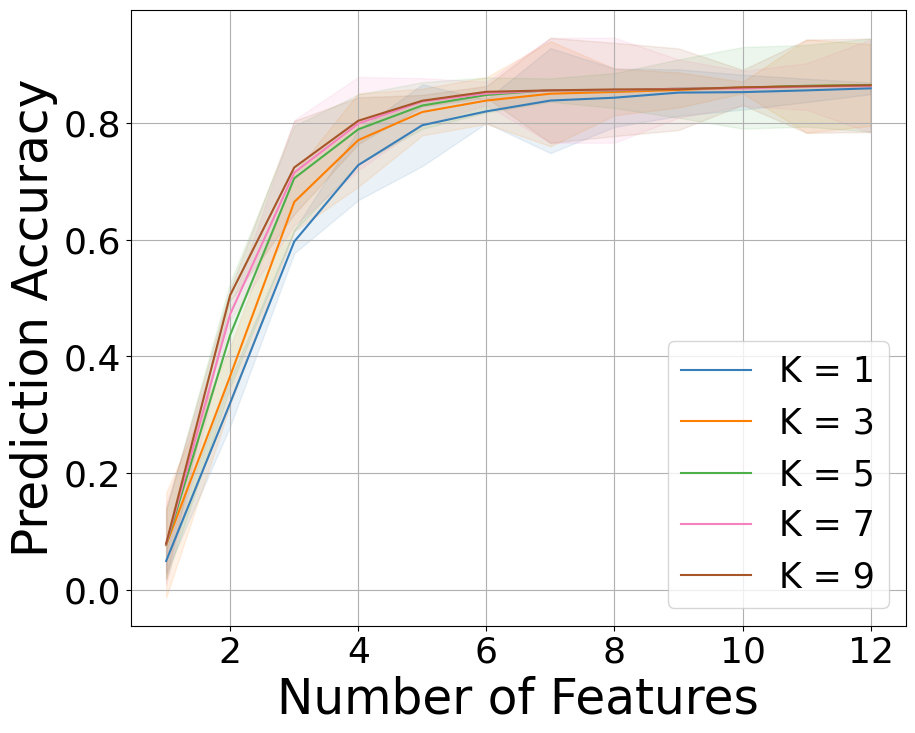}
        \caption{Stanford Cars}
        \label{fig:om_stanford}
    \end{subfigure}
    \hfill
    \begin{subfigure}[b]{0.19\textwidth}
        \centering
        \includegraphics[width=\textwidth]{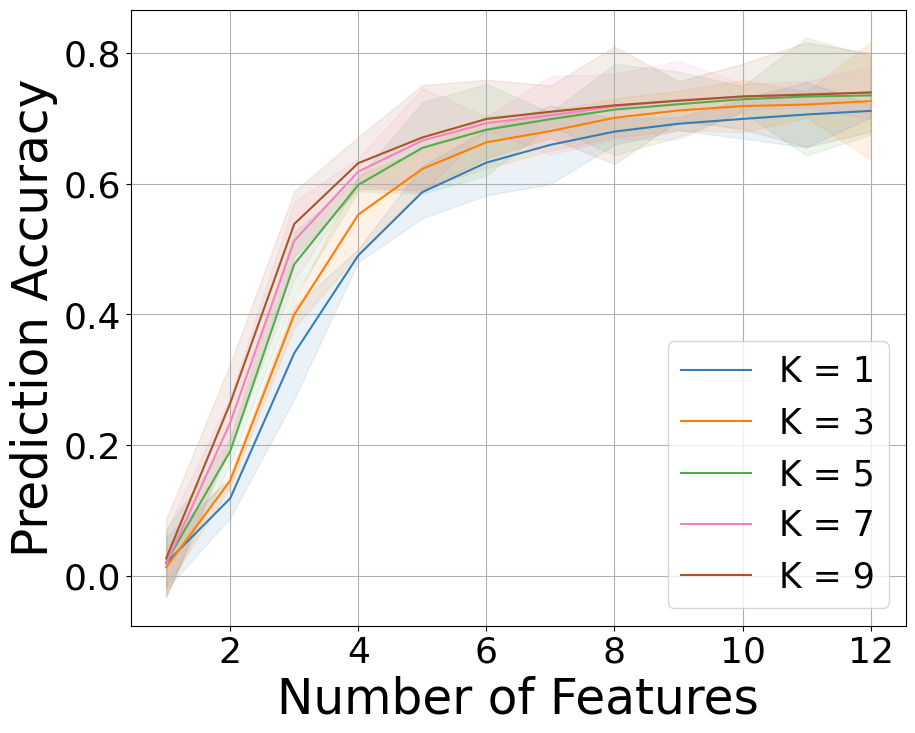}
        \caption{CIFAR-100}
        \label{fig:om_cifar100}
    \end{subfigure}
    \hfill

    \caption{The performance of \methodname\ with different choice of hyperparameters. Each of the curves corresponds to a value of $K$ ($K$ nearest neighbours), and the horizontal axis shows the number of features $M$ used for prediction.}
    \label{fig:sensitivity_analysis}
\end{figure}

\section{Conclusion}

\methodname\ demonstrates a novel form of interpretable machine learning, which performs decision-making through the similarity of the concepts within the test image to the corresponding concepts in the training set. We demonstrate that this allows both factual (\textit{Why did the model predict this class?}) and counterfactual analysis (\textit{How would the model explain the predictions if the alternative class was predicted?}). The experimental results show both a competitive accuracy of the method and demonstrate, empirically and theoretically, that the method has favourable properties of fidelity, necessity, sufficiency, and sparsity. Surprisingly, it also demonstrates impressive finetuning-free $k$-NN generalisation to new datasets.

It is also worth noting that the training dataset $\mathcal{D}$ from Algorithm \ref{alg:one} can be, without any changes in the method, be replaced with a trusted dataset, which can contain a private collection of data which the predictor could relate the testing image to. It can be motivated by both lack of access to training images or by lack of trust in the training data due to an inherent noise and can be especially useful in safety-critical domains such as medical imagery.  

While the current work only focuses on image classification, future work can expand this method to the segmentation scenarios. To reflect upon this, we demonstrate, in Appendix \ref{segmentation_section_appendix},  the potential for concept-based segmentation using \methodname. Such concept-based segmentation can enable human-in-the-loop decision-making: a human can change class-defining features in the model by selecting the corresponding segments. One can see the potential of the method to detect manipulated imaging and adversarial attacks by highlighting the areas of forgery and comparing them with training-set examples. We outline the limitations and broader impacts in Appendices \ref{broader_impacts} and \ref{limitations_section} respectively.

\section{Reproducibility Statement}

To ensure reproducibility of the results presented in this paper, we provide detailed information on the following key aspects:

\begin{itemize}
    \item \textbf{Datasets:} The experiments were conducted on well-known datasets, including CIFAR-10, CIFAR-100, CUB-200-2011, Stanford Cars, and Oxford-IIIT Pets. These datasets are publicly available and widely used in computer vision research.
    \item \textbf{Model Architectures:} We used $\operatorname{B-cos}$~(\cite{bohle2022b}) and ViT models, which are described in detail in both the main text and the Appendix. The model architecture, including any modifications for our method (\methodname), is fully explained, ensuring that the implementation can be reproduced by others. The training details are given in the Appendix \ref{bcos_hyperparameters}.
    \item \textbf{Training Details:} The model training process is described with exact hyperparameters provided in Appendix \ref{experimental_setup_section}. We also offer details on hardware used (e.g., GeForce RTX 2080 Ti with CUDA version 12.5) and software packages (e.g., NumPy, PyTorch, Torchvision), making it easy to replicate the experiments.
    \item \textbf{Evaluation Metrics:} All metrics used for evaluation, such as accuracy, fidelity, sparsity, and C-insertion/C-deletion metrics, are well-documented, ensuring consistency in reproducing the reported performance.
    
    \item \textbf{Code Availability:} To support reproducibility, we will provide the code used to conduct these experiments, including tables and analysis. This code, including all prompt templates and post-processing scripts, will be made publicly available upon publication.
\end{itemize}



\section{Acknowledgement}

This work was partially funded by ELSA – European Lighthouse on Secure and Safe AI funded by the European Union under grant agreement No. 101070617, PrivateAim from the German Federal Ministry of Education and Research under the funding code 01ZZ2316G and Image-Tox ZT-I-PF-4-037 supported by the impulse and networking fund of the Helmholtz Association. Views and opinions expressed are however those of the authors only and do not necessarily reflect those of the funding agencies and they can neither be held responsible for them. 

Dmitry Kangin acknowledges travel support from the European Union’s Horizon 2020 research and innovation programme under grant agreement No 951847.

\bibliography{iclr2025_conference}
\bibliographystyle{iclr2025_conference}

\newpage
\appendix

\section{B-cos Networks}
\label{sec:Appendix_BCos_description}

$\operatorname{B-cos}$ networks offer a novel approach to improving the interpretability of deep learning models by ensuring that input features align with the model's weights throughout training. This innovation arises from the realization that although deep neural networks excel in performance across various tasks, their internal workings remain largely opaque and hard to interpret. Typically, deep models rely on linear transformations coupled with non-linear activations, a design that contributes to their "black box" nature.

In contrast, $\operatorname{B-cos}$ networks replace traditional linear transformations with the $\operatorname{B-cos}$ transformation, which promotes alignment between inputs and weights. This transformation is defined as
\begin{equation}
\operatorname{B-cos}(\mathbf{x}; \mathbf{w}) = \|\mathbf{w}\| \|\mathbf{x}\| \cos^B(\theta) \cdot \text{sign}(\cos(\theta)),
\end{equation}
where \( \theta \) denotes the angle between the input vector \( \mathbf{x} \) and the weight vector \( \mathbf{w} \), and \( B \) is a hyperparameter that amplifies the model's sensitivity to alignment. This transformation shifts the model's focus from merely achieving high performance to fostering interpretability by emphasizing the relationship between the input data and model features.

The training of $\operatorname{B-cos}$ networks integrates this alignment directly into the optimization process. By applying alignment pressure during weight adjustment, $\operatorname{B-cos}$ networks encourage the model to align its weights with the most relevant input features, making this alignment a key objective rather than a byproduct of training, which is a departure from conventional methods focused solely on minimizing prediction error.

The integration of $\operatorname{B-cos}$ transformations into existing architectures is seamless since they can serve as direct replacements for typical linear layers. This compatibility enables the application of $\operatorname{B-cos}$ networks to a broad array of architectures such as VGG, ResNet, InceptionNet, DenseNet, and Vision Transformers (ViTs) ~\cite{bohle2024b}, without significant changes to their core structure. Empirical results demonstrate that this transformation maintains competitive performance on standard datasets like ImageNet while enhancing model interpretability.

During inference, a key advantage of $\operatorname{B-cos}$ networks becomes apparent. The sequence of $\operatorname{B-cos}$ transformations throughout the network simplifies into a single linear operation, as the successive alignment-focused layers collapse into a single transformation. Mathematically, this is expressed as:
\begin{equation}
W_{1\rightarrow L}(\mathbf{x}) = W_1 \times W_2 \times \ldots \times W_L,
\end{equation}
where \( W_{1\rightarrow L}(\mathbf{x}) \) represents the effective weight matrix over all \( L \) layers, and \( W_1, W_2, \ldots, W_L \) are the weight matrices of the individual $\operatorname{B-cos}$ layers. This reduces the network’s computation at test time to:
\begin{equation}
y =  W_{1\rightarrow L}(\mathbf{x}) \cdot \mathbf{x},
\end{equation}

where \( y \) is the output. This reduction to a single matrix-vector multiplication significantly improves both computational efficiency and transparency, offering a clear view of how input features affect the output. The network's behavior, represented by $\theta_{1\rightarrow L}$ in the main text, becomes fully interpretable, as emphasized in~\cite{bohle2022b}.

\section{Experimental Setup}
\label{experimental_setup_section}
\subsection{Hardware and software configuration}
We trained and tested our models using GeForce RTX 2080 Ti with CUDA version 12.5. In our work we use the following software packages for Python:
\begin{enumerate}
\item NumPy 1.26.2
\item PyTorch 2.1.2
\item Torchvision 0.16.2
\item Matplotlib 3.8.2
\end{enumerate}

\subsection{Model Training and Evaluation Details}

Except for zero-shot learning settings, the $\operatorname{B-cos}$ models has been trained with the  hyperparameters outlined in Table \ref{bcos_hyperparameters}. For more details on the meaning of  the training parameters for the B-cos model, please follow the work \cite{bohle2024b}.

For the pretrained baselines and models, we use the models from the open sources, which can be downloaded from the following repository: \href{https://github.com/B-cos/B-cos-v2}{https://github.com/B-cos/B-cos-v2}. For $\operatorname{B-cos/ViT}$, we use \texttt{vitc\_l\_patch1\_14} model  pretrained on ImageNet

\begin{table}
\centering
 \caption{Hyperparameters of the $\operatorname{B-cos}$ model training}
 \begin{tabular}{ccc} 
 \toprule
Hyperparameter &  Value & Comments\\ [0.5ex] 
\midrule
Learning rate & $0.01$ & \\ 
Max \# epochs & $500$ & with early stopping \\ 
Batch size & $16$ & \\ 
Drop out value  & $0.5$ &  \\ 
B-value in $\operatorname{B-cos}$ & $1.5$ &  \\ 
Loss & BCE Loss & \\ 
Final activation & sigmoid & \\ 
 \bottomrule
 \end{tabular}
 \label{bcos_hyperparameters}
\end{table}

\section{Non-finetuned model performance}
\label{performance_nf}

Table \ref{filtered_accuracy_eval_datasets} demonstrates comparative performance between the finetuned and non-finetuned method. 
\begin{table}[]
    \centering
 \caption{Comparison between the performance of  finetuned and non-finetuned model based on  $\operatorname{B-cos/Vit}$ backbone}
\begin{tabular}{lcc} 
 \toprule
\textbf{Dataset} & \textbf{\methodname} & \textbf{\methodname} (no finetuning) \\ [0.5ex] 
\midrule
 Oxford-IIIT Pets & $87.73$ & $85.52 $ \\ 
 CUB-200-2011 & $74.14$ & $70.04 $  \\ 
 Stanford Cars & $86.81$ & $84.03$ \\ 
 CIFAR-10 & $91.21$ & $90.04 $ \\ 
 CIFAR-100 & $76.42$ & $72.92 $ \\ 
 \bottomrule
\end{tabular}
\label{filtered_accuracy_eval_datasets}
\end{table}

\newpage

\section{Further details on hyperparameter choice}
\label{Appendix_hyperparameter_choice}

In Figure \ref{fig:sensitivity_analysis_2} we present the hyperparameter sensitivity analysis graphs for the $\ell^2$ distances between the whole set of CDFs. The experimental scheme is given in Figure \ref{fig:method_2}.

\begin{figure}
    \includegraphics[width=\textwidth]{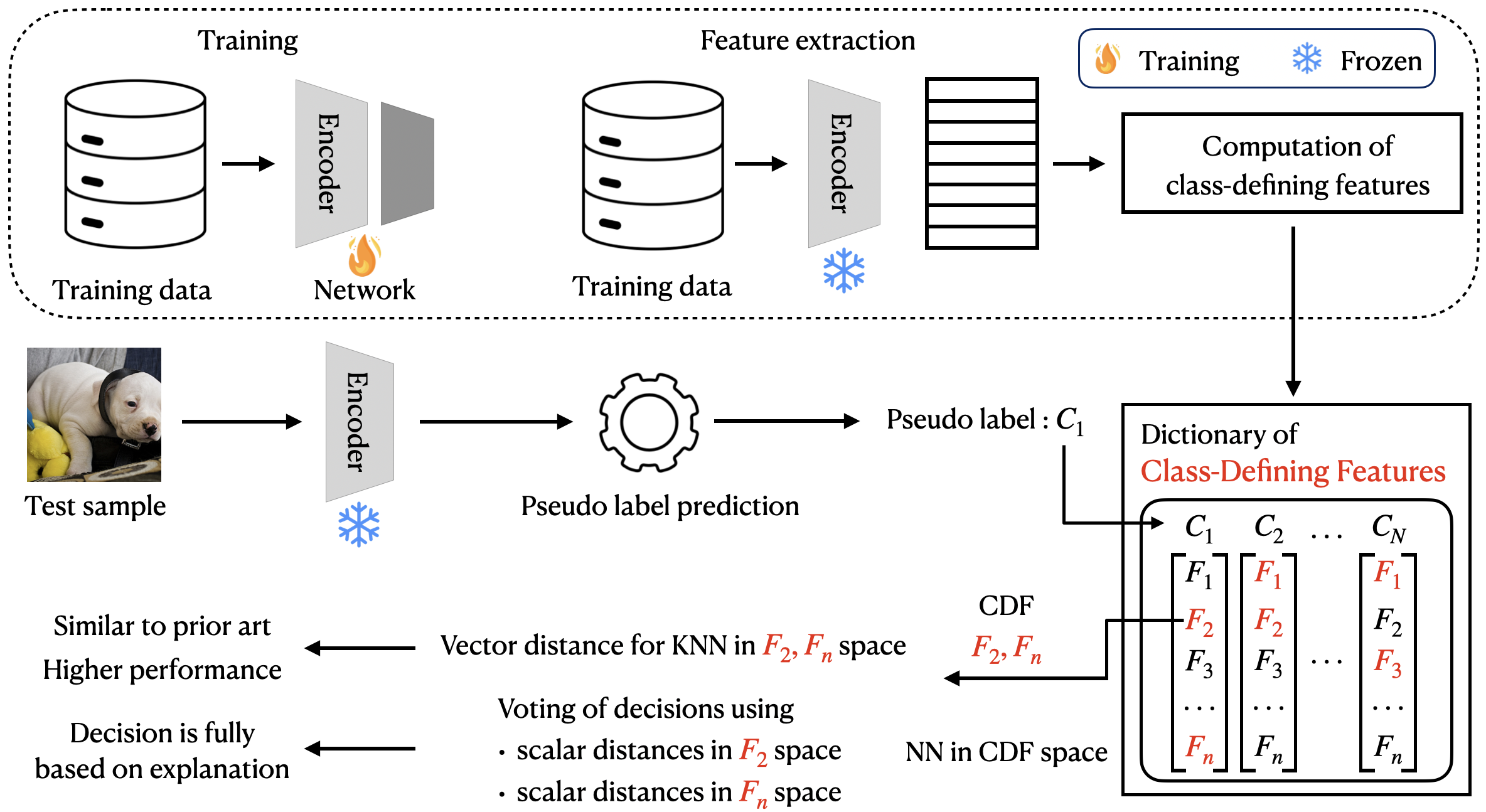}
    \caption{A method ablation for for sensitivity analysis, which uses the $\ell^2$ distances}
    \label{fig:method_2}
\end{figure}

\begin{figure}
    \centering
    \begin{subfigure}[b]{0.25\textwidth}
        \centering
        \includegraphics[width=\textwidth]{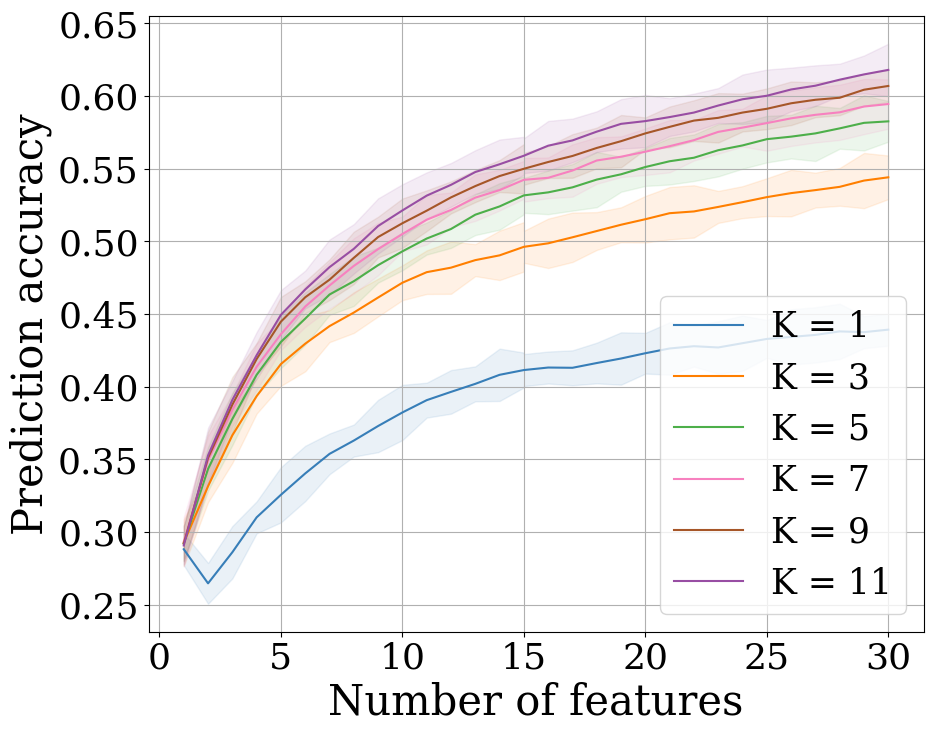}
        \caption{CUBS}
        \label{fig:om_cubs2}
    \end{subfigure}
    \hfill
    \begin{subfigure}[b]{0.25\textwidth}
        \centering
        
        \includegraphics[width=\textwidth]{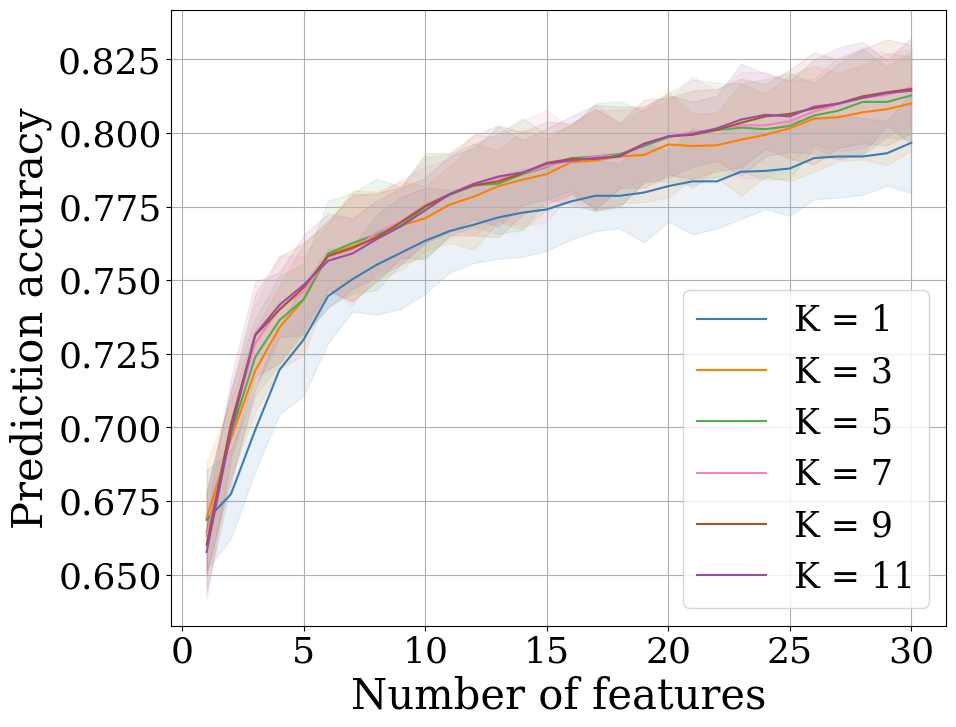}
        \caption{Oxford-IIIT Pets}
        \label{fig:om_pets_2}
    \end{subfigure}
    \hfill
    \begin{subfigure}[b]{0.25\textwidth}
        \centering
        \includegraphics[width=\textwidth]{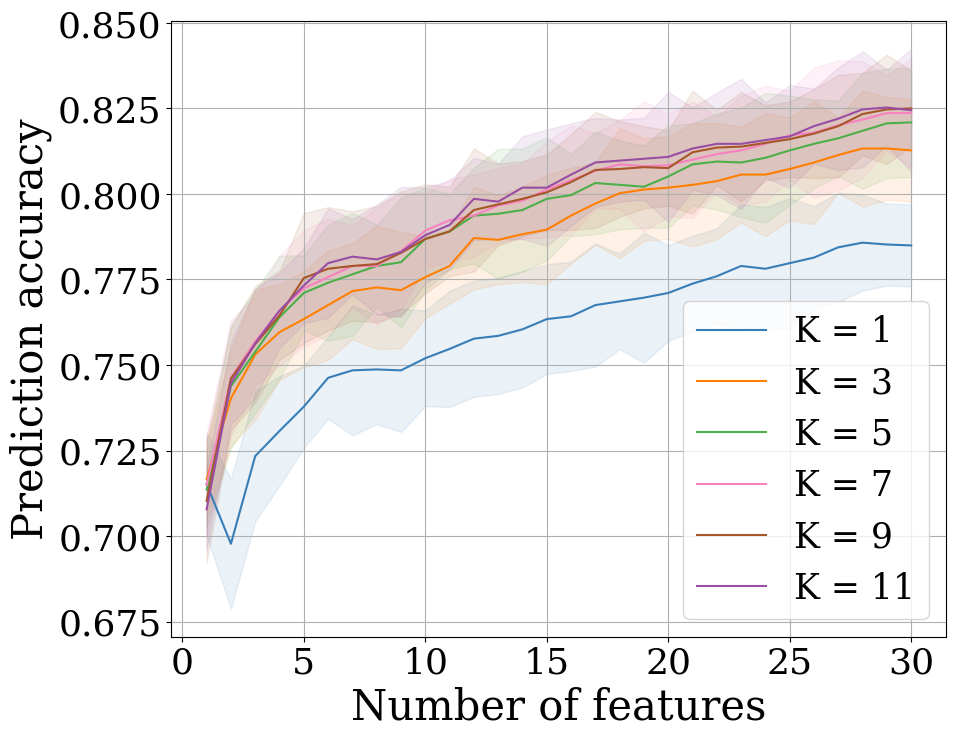}
        \caption{Stanford-Cars}
        \label{fig:om_cars_2}
    \end{subfigure}
    \hfill
    \begin{subfigure}[b]{0.25\textwidth}
        \centering
        \includegraphics[width=\textwidth]{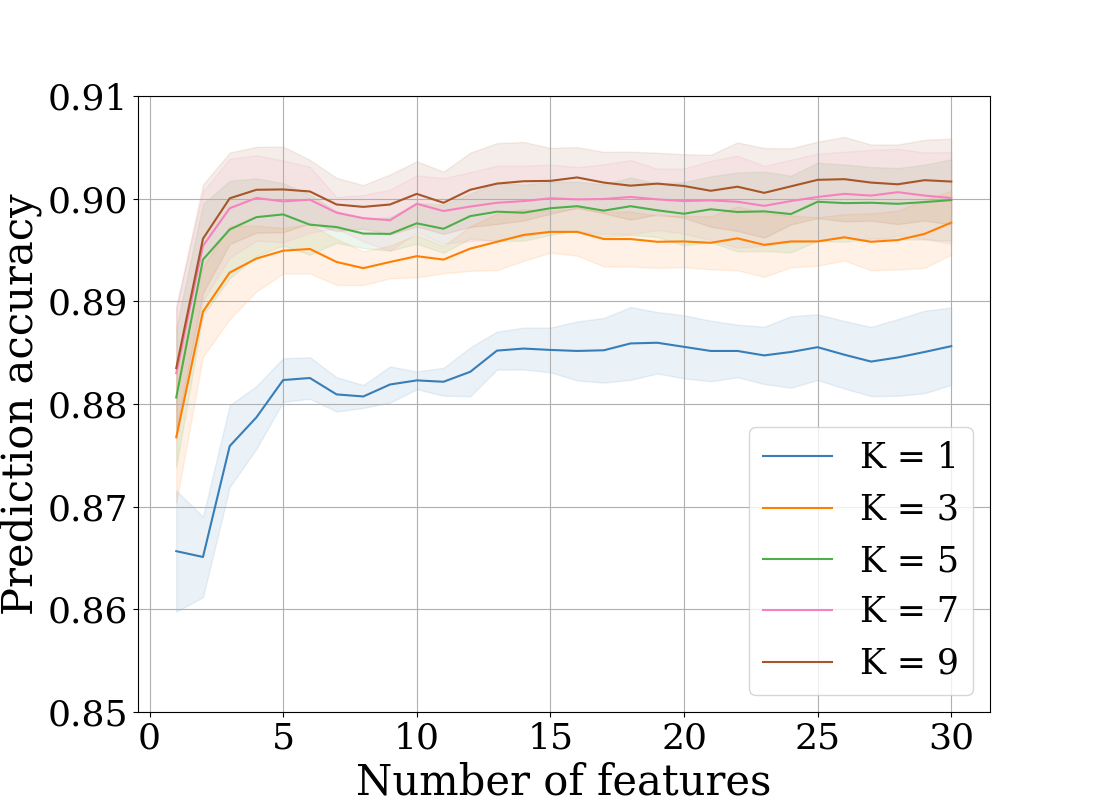}
        \caption{CIFAR-10}
        \label{fig:om_cifar10_2}
    \end{subfigure}
    \hfill
    \begin{subfigure}[b]{0.25\textwidth}
        \centering
        \includegraphics[width=\textwidth]{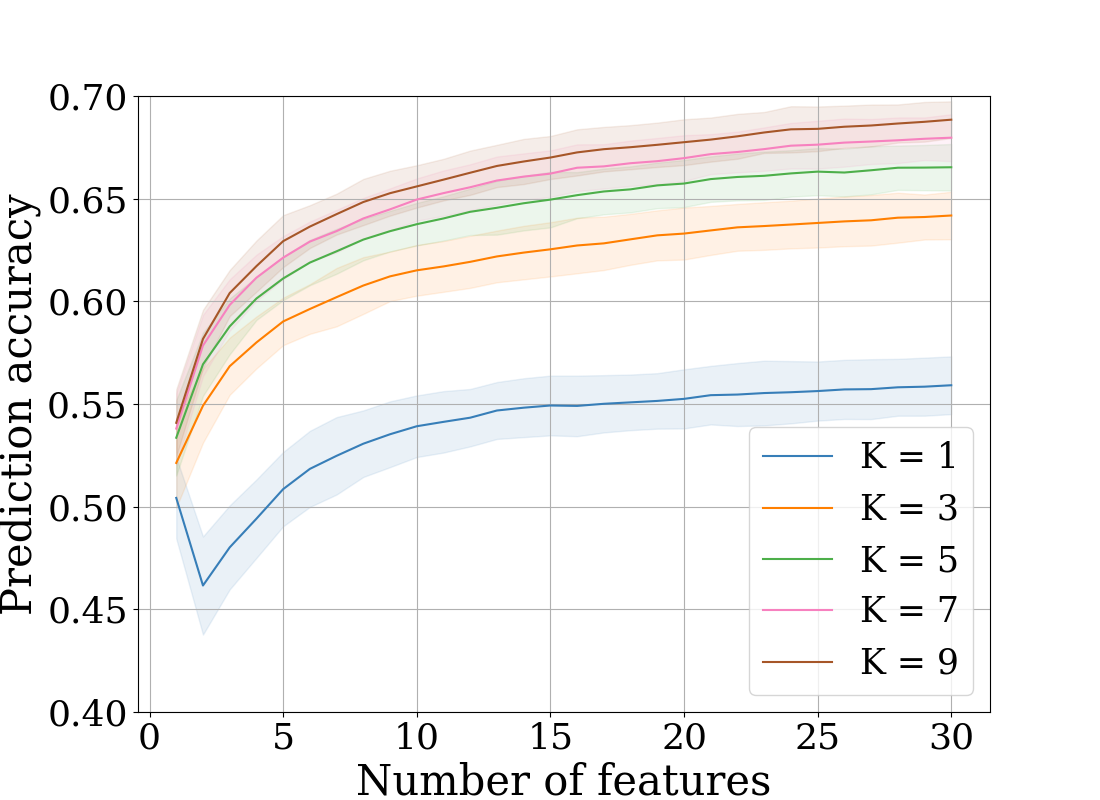}
        \caption{CIFAR 100}
        \label{fig:om_cifar100_2}
    \end{subfigure}
    \hfill
    \begin{subfigure}[b]{0.25\textwidth}
        \centering
        \includegraphics[width=\textwidth]{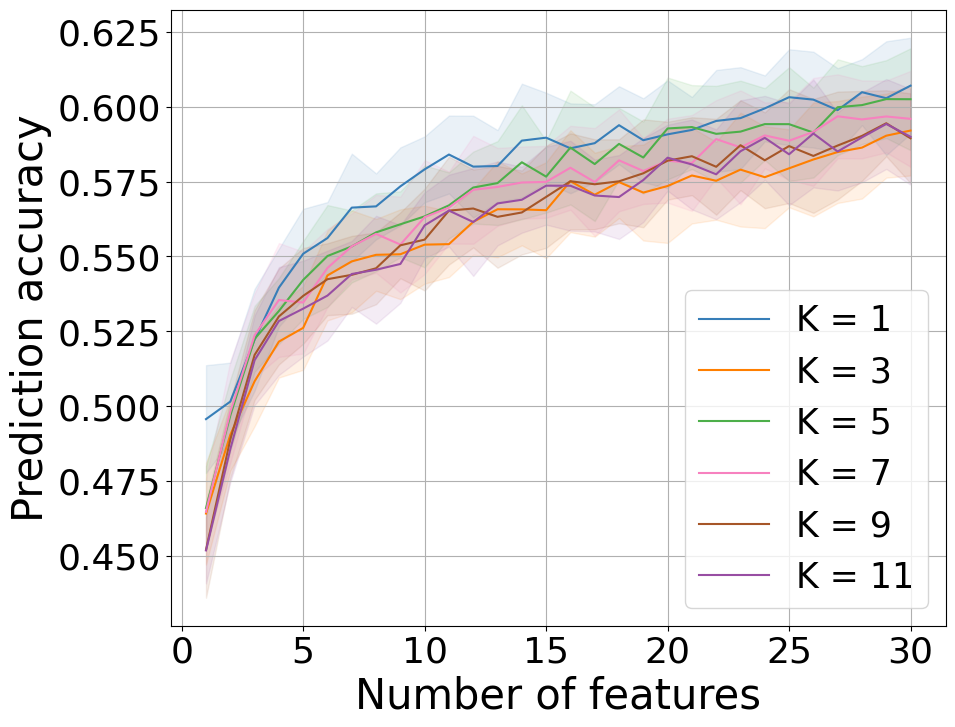}
        \caption{ImageNet}
        \label{fig:om_imagenet_2}
    \end{subfigure}
    
    \caption{The performance of \methodname\ with different choices of hyperparameters. Each of the lines corresponds to a value of $K$ ($K$ nearest neighbours), and the horizontal axis shows the number of features $M$ used for prediction.}
    \label{fig:sensitivity_analysis_2}
\end{figure}

\newpage

\section{Additional Qualitative Results}
\label{additional_qualitative_results}
%
    
%
    

\begin{figure}[h]
\begin{subfigure}{\textwidth}
\resizebox{\textwidth}{!}{
\begin{tikzpicture}[domain=0:15]
\node at (-3,1.25) {Testing image};
\node at (-3, 0) {\includegraphics[height=2cm]{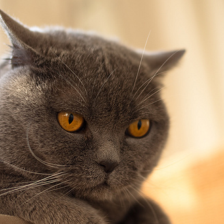}};

\node at (-3,-1.25) {Class-defining feature interpretations (testing image)};
\node at (-10, -2.5) {\includegraphics[height=2cm]{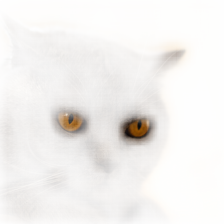}};
\node at (-8, -2.5) {\includegraphics[height=2cm]{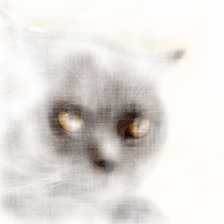}};
\node at (-6, -2.5) {\includegraphics[height=2cm]{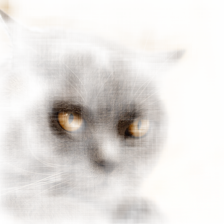}};
\node at (-4, -2.5) {\includegraphics[height=2cm]{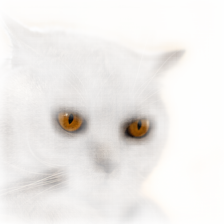}};
\node at (-2, -2.5) {\includegraphics[height=2cm]{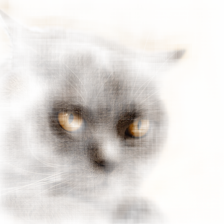}};
\node at (0, -2.5) {\includegraphics[height=2cm]{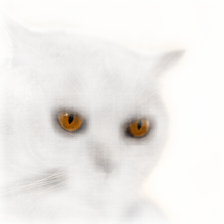}};
\node at (2, -2.5) {\includegraphics[height=2cm]{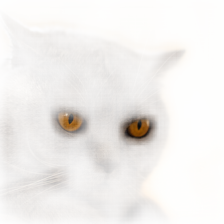}};
\node at (4, -2.5) {\includegraphics[height=2cm]{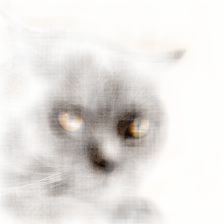}};
\draw[<->, line width=0.5mm] (-10,-3.5) -- (-10,-4);
\draw[<->, line width=0.5mm] (-8,-3.5) -- (-8,-4);
\draw[<->, line width=0.5mm] (-6,-3.5) -- (-6,-4);
\draw[<->, line width=0.5mm] (-4,-3.5) -- (-4,-4);
\draw[<->, line width=0.5mm] (-2,-3.5) -- (-2,-4);
\draw[<->, line width=0.5mm] (0,-3.5)  -- (0,-4);
\draw[<->, line width=0.5mm] (2,-3.5)  -- (2,-4);
\draw[<->, line width=0.5mm] (4,-3.5)  -- (4,-4);

\node at (-3,-6.25) {Class-defining feature interpretations (training images)};
\node at (-10, -5){\includegraphics[height=2cm]{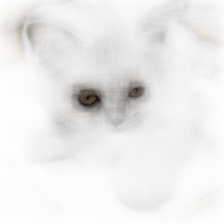}};
\node at (-8, -5){\includegraphics[height=2cm]{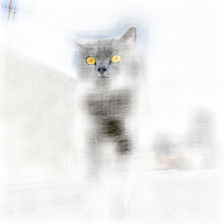}};
\node at (-6, -5){\includegraphics[height=2cm]{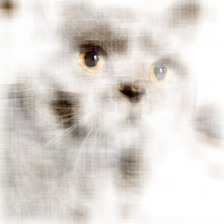}};
\node at (-4, -5){\includegraphics[height=2cm]{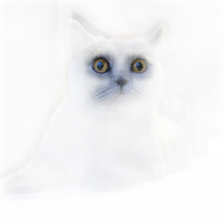}};
\node at (-2, -5){\includegraphics[height=2cm]{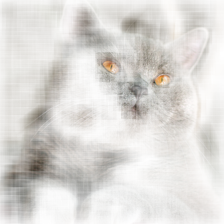}};
\node at (0, -5){\includegraphics[height=2cm]{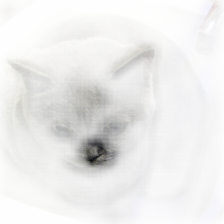}};
\node at (2, -5){\includegraphics[height=2cm]{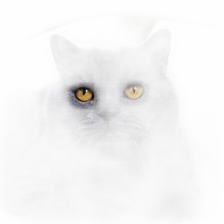}};
\node at (4, -5){\includegraphics[height=2cm]{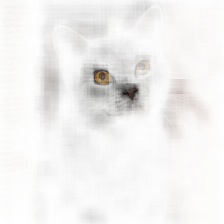}};

\node at (-10,-9) {\makecell{British\\ shorthair}};
\node at (-8,-9) {\makecell{British\\ shorthair}};
\node at (-6,-9) {\makecell{British\\ shorthair}};
\node at (-4,-9) {\makecell{British\\ shorthair}};
\node at (-2,-9) {\makecell{British\\ shorthair}};
\node at (0,-9) {\makecell{British\\ shorthair}};
\node at (2,-9) {\makecell{British\\ shorthair}};
\node at (4,-9) {\makecell{British\\ shorthair}};

\node at (-3,-9.75) {Training data};

\node at (-10, -7.5) {\includegraphics[height=2cm]{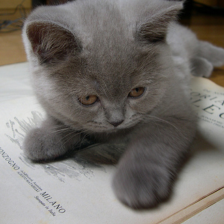}};
\node at (-8, -7.5) {\includegraphics[height=2cm]{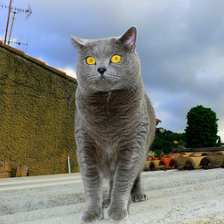}};
\node at (-6, -7.5) {\includegraphics[height=2cm]{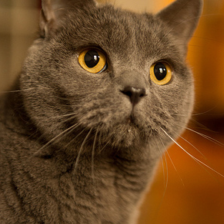}};
\node at (-4, -7.5) {\includegraphics[height=2cm]{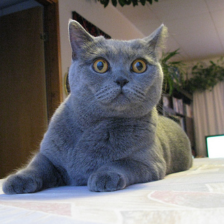}};
\node at (-2, -7.5) {\includegraphics[height=2cm]{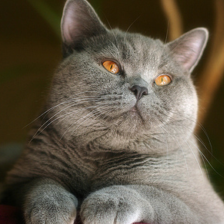}};
\node at (0, -7.5) {\includegraphics[height=2cm]{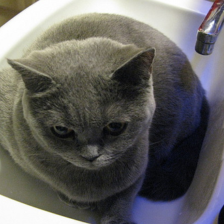}};
\node at (2, -7.5) {\includegraphics[height=2cm]{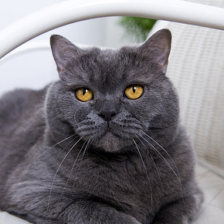}};
\node at (4, -7.5) {\includegraphics[height=2cm]{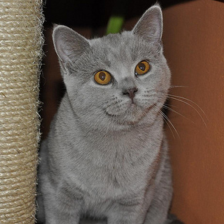}};

\end{tikzpicture}
}
\end{subfigure}
\begin{subfigure}{\textwidth}
\resizebox{\textwidth}{!}{
\begin{tikzpicture}[domain=0:15]
\node at (-3,1.25) {Testing image};
\node at (-3, 0) {\includegraphics[height=2cm]{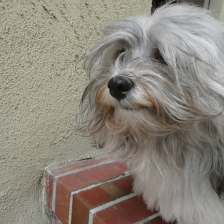}};

\node at (-3,-1.25) {Class-defining feature interpretations (testing image)};
\node at (-10, -2.5) {\includegraphics[height=2cm]{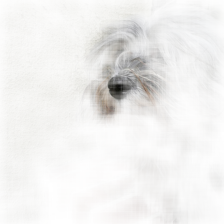}};
\node at (-8, -2.5) {\includegraphics[height=2cm]{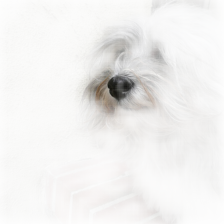}};
\node at (-6, -2.5) {\includegraphics[height=2cm]{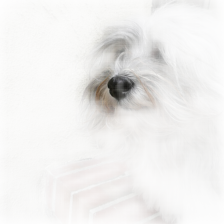}};
\node at (-4, -2.5) {\includegraphics[height=2cm]{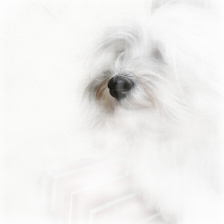}};
\node at (-2, -2.5) {\includegraphics[height=2cm]{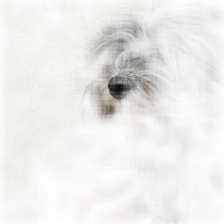}};
\node at (0, -2.5) {\includegraphics[height=2cm]{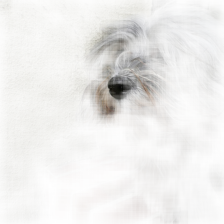}};
\node at (2, -2.5) {\includegraphics[height=2cm]{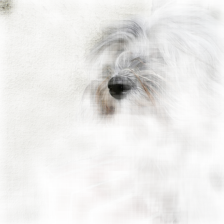}};
\node at (4, -2.5) {\includegraphics[height=2cm]{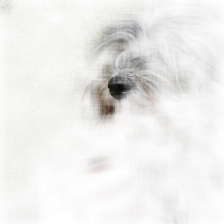}};
\draw[<->, line width=0.5mm] (-10,-3.5) -- (-10,-4);
\draw[<->, line width=0.5mm] (-8,-3.5) -- (-8,-4);
\draw[<->, line width=0.5mm] (-6,-3.5) -- (-6,-4);
\draw[<->, line width=0.5mm] (-4,-3.5) -- (-4,-4);
\draw[<->, line width=0.5mm] (-2,-3.5) -- (-2,-4);
\draw[<->, line width=0.5mm] (0,-3.5)  -- (0,-4);
\draw[<->, line width=0.5mm] (2,-3.5)  -- (2,-4);
\draw[<->, line width=0.5mm] (4,-3.5)  -- (4,-4);

\node at (-3,-6.25) {Class-defining feature interpretations (training images)};
\node at (-10, -5){\includegraphics[height=2cm]{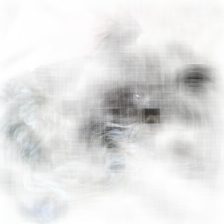}};
\node at (-8, -5){\includegraphics[height=2cm]{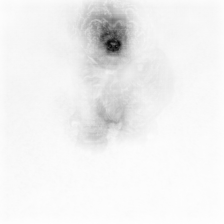}};
\node at (-6, -5){\includegraphics[height=2cm]{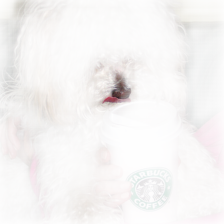}};
\node at (-4, -5){\includegraphics[height=2cm]{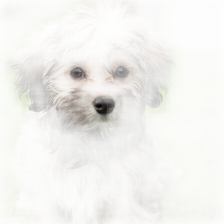}};
\node at (-2, -5){\includegraphics[height=2cm]{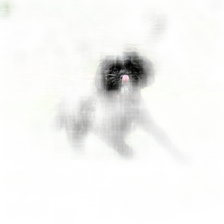}};
\node at (0, -5){\includegraphics[height=2cm]{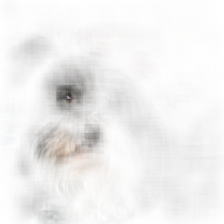}};
\node at (2, -5){\includegraphics[height=2cm]{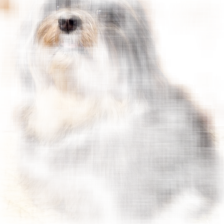}};
\node at (4, -5){\includegraphics[height=2cm]{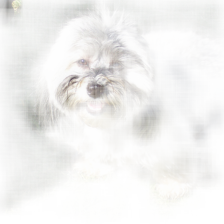}};

\node at (-10,-8.75) {\makecell{Havanese}};
\node at (-8,-8.75) {\makecell{Havanese}};
\node at (-6,-8.75) {\makecell{Havanese}};
\node at (-4,-8.75) {\makecell{Havanese}};
\node at (-2,-8.75) {\makecell{Havanese}};
\node at (0,-8.75) {\makecell{Havanese}};
\node at (2,-8.75) {\makecell{Havanese}};
\node at (4,-8.75) {\makecell{Havanese}};

\node at (-3,-9.25) {Training data};

\node at (-10, -7.5) {\includegraphics[height=2cm]{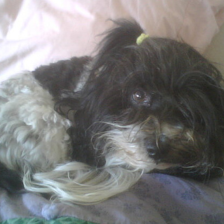}};
\node at (-8, -7.5) {\includegraphics[height=2cm]{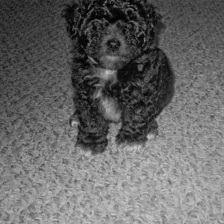}};
\node at (-6, -7.5) {\includegraphics[height=2cm]{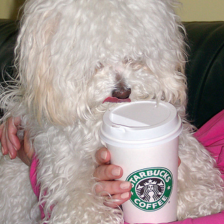}};
\node at (-4, -7.5) {\includegraphics[height=2cm]{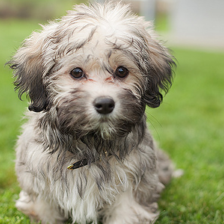}};
\node at (-2, -7.5) {\includegraphics[height=2cm]{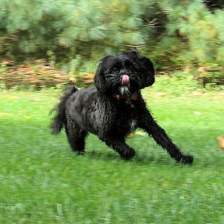}};
\node at (0, -7.5) {\includegraphics[height=2cm]{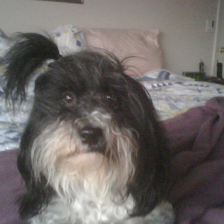}};
\node at (2, -7.5) {\includegraphics[height=2cm]{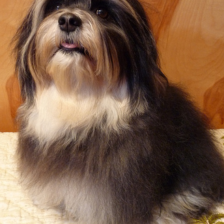}};
\node at (4, -7.5) {\includegraphics[height=2cm]{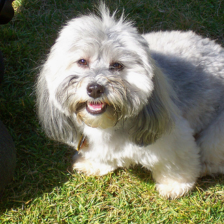}};

\end{tikzpicture}
}
\end{subfigure}
\caption{Additional qualitative results (Oxford-IIIT Pets)}
    \label{qualitative_results_appendix}
\end{figure}

\begin{figure}[h]
\begin{subfigure}{\textwidth}
\resizebox{\textwidth}{!}{
\begin{tikzpicture}[domain=0:15]
\node at (-3,1.25) {Testing image};
\node at (-3, 0) {\includegraphics[height=2cm]{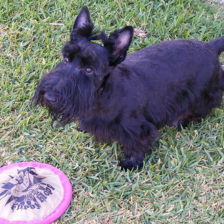}};

\node at (-3,-1.25) {Class-defining feature interpretations (testing image)};
\node at (-10, -2.5) {\includegraphics[height=2cm]{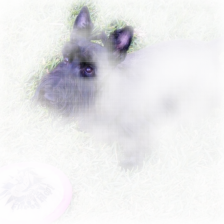}};
\node at (-8, -2.5) {\includegraphics[height=2cm]{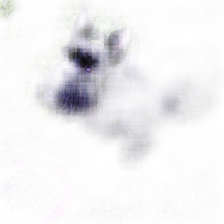}};
\node at (-6, -2.5) {\includegraphics[height=2cm]{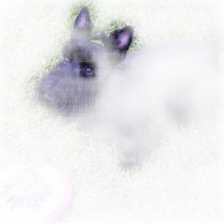}};
\node at (-4, -2.5) {\includegraphics[height=2cm]{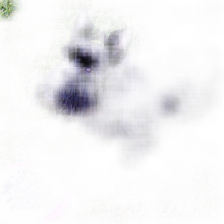}};
\node at (-2, -2.5) {\includegraphics[height=2cm]{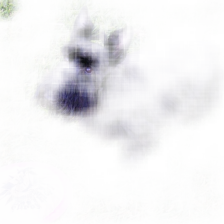}};
\node at (0, -2.5) {\includegraphics[height=2cm]{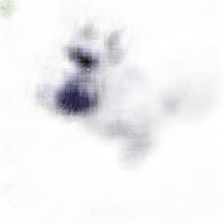}};
\node at (2, -2.5) {\includegraphics[height=2cm]{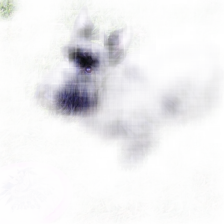}};
\node at (4, -2.5) {\includegraphics[height=2cm]{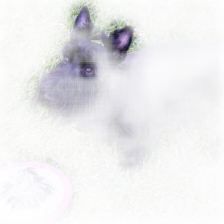}};
\draw[<->, line width=0.5mm] (-10,-3.5) -- (-10,-4);
\draw[<->, line width=0.5mm] (-8,-3.5) -- (-8,-4);
\draw[<->, line width=0.5mm] (-6,-3.5) -- (-6,-4);
\draw[<->, line width=0.5mm] (-4,-3.5) -- (-4,-4);
\draw[<->, line width=0.5mm] (-2,-3.5) -- (-2,-4);
\draw[<->, line width=0.5mm] (0,-3.5)  -- (0,-4);
\draw[<->, line width=0.5mm] (2,-3.5)  -- (2,-4);
\draw[<->, line width=0.5mm] (4,-3.5)  -- (4,-4);

\node at (-3,-6.25) {Class-defining feature interpretations (training images)};
\node at (-10, -5){\includegraphics[height=2cm]{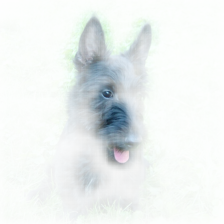}};
\node at (-8, -5){\includegraphics[height=2cm]{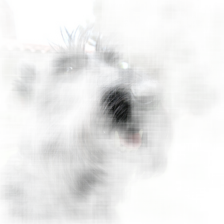}};
\node at (-6, -5){\includegraphics[height=2cm]{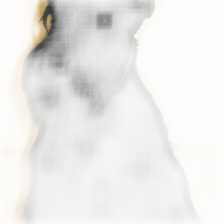}};
\node at (-4, -5){\includegraphics[height=2cm]{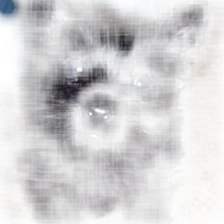}};
\node at (-2, -5){\includegraphics[height=2cm]{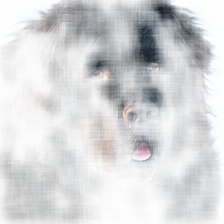}};
\node at (0, -5){\includegraphics[height=2cm]{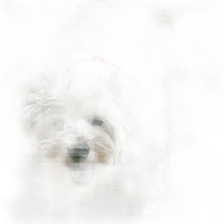}};
\node at (2, -5){\includegraphics[height=2cm]{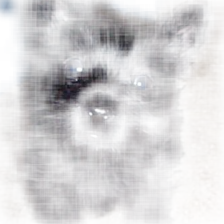}};
\node at (4, -5){\includegraphics[height=2cm]{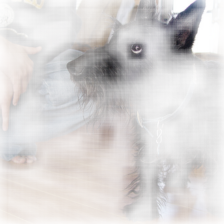}};

\node at (-10,-9) {\makecell{Scottish\\Terrier}};
\node at (-8,-9) {\makecell{Scottish\\Terrier}};
\node at (-6,-9) {\makecell{Scottish\\Terrier}};
\node at (-4,-9) {\makecell{Scottish\\Terrier}};
\node at (-2,-9) {\makecell{Newfoundland}};
\node at (0,-9) {\makecell{Scottish\\Terrier}};
\node at (2,-9) {\makecell{Scottish\\Terrier}};
\node at (4,-9) {\makecell{Scottish\\Terrier}};

\node at (-3,-9.75) {Training data};

\node at (-10, -7.5) {\includegraphics[height=2cm]{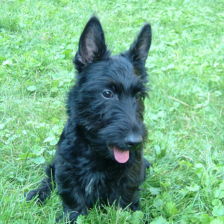}};
\node at (-8, -7.5) {\includegraphics[height=2cm]{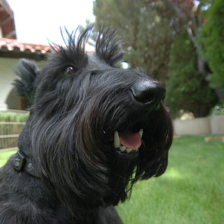}};
\node at (-6, -7.5) {\includegraphics[height=2cm]{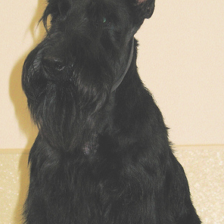}};
\node at (-4, -7.5) {\includegraphics[height=2cm]{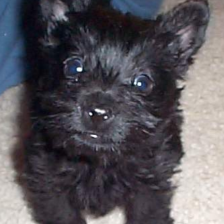}};
\node at (-2, -7.5) {\includegraphics[height=2cm]{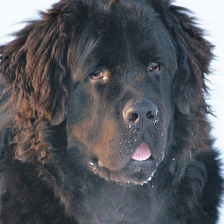}};
\node at (0, -7.5) {\includegraphics[height=2cm]{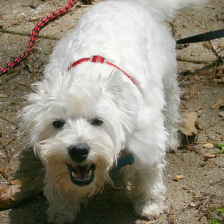}};
\node at (2, -7.5) {\includegraphics[height=2cm]{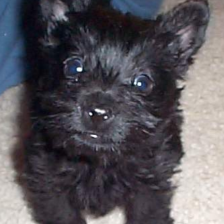}};
\node at (4, -7.5) {\includegraphics[height=2cm]{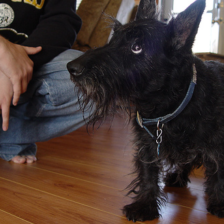}};

\draw [draw=red, line width=0.5mm] (-3,-6.5) rectangle (-1,-8.5);
\draw [draw=red, line width=0.5mm] (-3,-4) rectangle (-1,-6);

\end{tikzpicture}
}
\end{subfigure}
\begin{subfigure}{\textwidth}
\resizebox{\textwidth}{!}{
\begin{tikzpicture}[domain=0:15]
\node at (-3,1.25) {Testing image};
\node at (-3, 0) {\includegraphics[height=2cm]{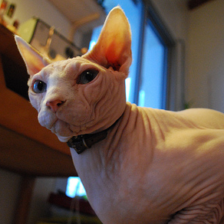}};

\node at (-3,-1.25) {Class-defining feature interpretations (testing image)};
\node at (-10, -2.5) {\includegraphics[height=2cm]{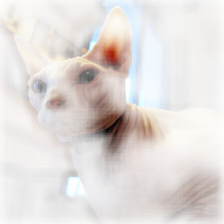}};
\node at (-8, -2.5) {\includegraphics[height=2cm]{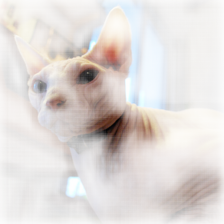}};
\node at (-6, -2.5) {\includegraphics[height=2cm]{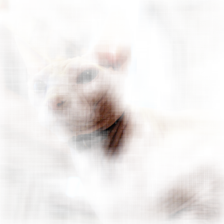}};
\node at (-4, -2.5) {\includegraphics[height=2cm]{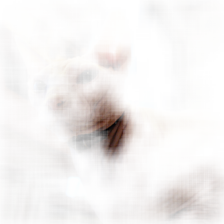}};
\node at (-2, -2.5) {\includegraphics[height=2cm]{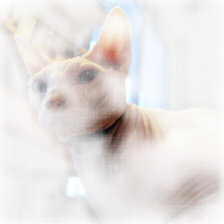}};
\node at (0, -2.5) {\includegraphics[height=2cm]{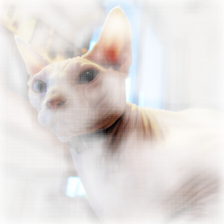}};
\node at (2, -2.5) {\includegraphics[height=2cm]{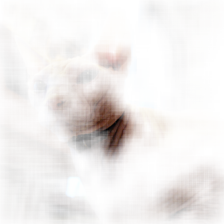}};
\node at (4, -2.5) {\includegraphics[height=2cm]{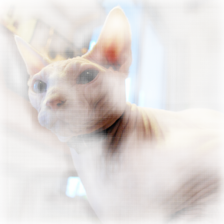}};
\draw[<->, line width=0.5mm] (-10,-3.5) -- (-10,-4);
\draw[<->, line width=0.5mm] (-8,-3.5) -- (-8,-4);
\draw[<->, line width=0.5mm] (-6,-3.5) -- (-6,-4);
\draw[<->, line width=0.5mm] (-4,-3.5) -- (-4,-4);
\draw[<->, line width=0.5mm] (-2,-3.5) -- (-2,-4);
\draw[<->, line width=0.5mm] (0,-3.5)  -- (0,-4);
\draw[<->, line width=0.5mm] (2,-3.5)  -- (2,-4);
\draw[<->, line width=0.5mm] (4,-3.5)  -- (4,-4);

\node at (-3,-6.25) {Class-defining feature interpretations (training images)};
\node at (-10, -5){\includegraphics[height=2cm]{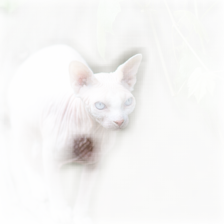}};
\node at (-8, -5){\includegraphics[height=2cm]{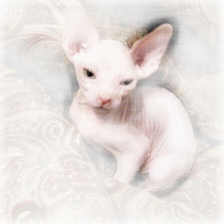}};
\node at (-6, -5){\includegraphics[height=2cm]{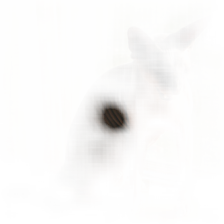}};
\node at (-4, -5){\includegraphics[height=2cm]{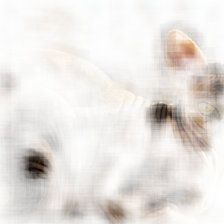}};
\node at (-2, -5){\includegraphics[height=2cm]{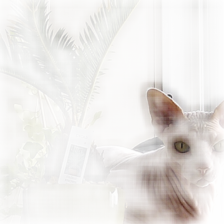}};
\node at (0, -5){\includegraphics[height=2cm]{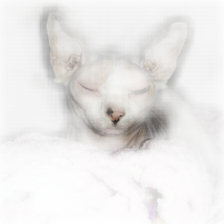}};
\node at (2, -5){\includegraphics[height=2cm]{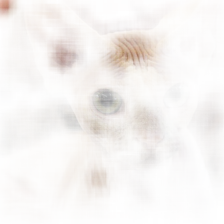}};
\node at (4, -5){\includegraphics[height=2cm]{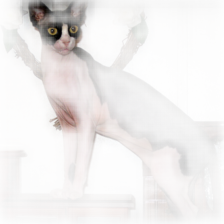}};

\node at (-10,-8.75) {\makecell{Sphynx}};
\node at (-8,-8.75) {\makecell{Sphynx}};
\node at (-6,-8.75) {\makecell{Sphynx}};
\node at (-4,-8.75) {\makecell{Sphynx}};
\node at (-2,-8.75) {\makecell{Sphynx}};
\node at (0,-8.75) {\makecell{Sphynx}};
\node at (2,-8.75) {\makecell{Sphynx}};
\node at (4,-8.75) {\makecell{Sphynx}};

\node at (-3,-9.25) {Training data};

\node at (-10, -7.5) {\includegraphics[height=2cm]{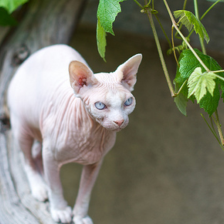}};
\node at (-8, -7.5) {\includegraphics[height=2cm]{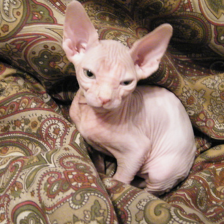}};
\node at (-6, -7.5) {\includegraphics[height=2cm]{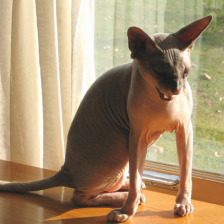}};
\node at (-4, -7.5) {\includegraphics[height=2cm]{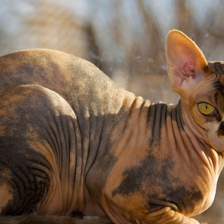}};
\node at (-2, -7.5) {\includegraphics[height=2cm]{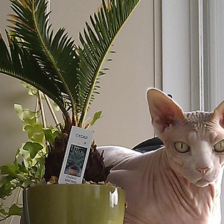}};
\node at (0, -7.5) {\includegraphics[height=2cm]{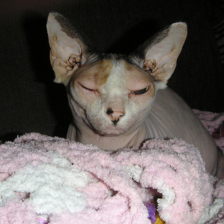}};
\node at (2, -7.5) {\includegraphics[height=2cm]{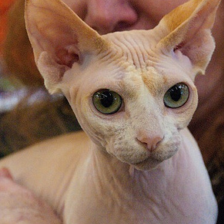}};
\node at (4, -7.5) {\includegraphics[height=2cm]{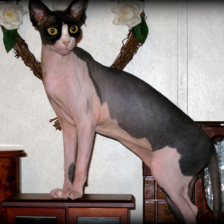}};

\end{tikzpicture}
}
\end{subfigure}
\caption{Additional qualitative results  (Oxford-IIIT Pets)}
    \label{qualitative_results_appendix2}
\end{figure}

\begin{figure}[h]
\begin{subfigure}{\textwidth}
\resizebox{\textwidth}{!}{
\begin{tikzpicture}[domain=0:15]
\node at (-3,1.25) {Testing image};
\node at (-3, 0) {\includegraphics[height=2cm]{figures_cub/1946/test_image.png}};

\node at (-3,-1.25) {Class-defining feature interpretations (testing image)};
\node at (-10, -2.5) {\includegraphics[height=2cm]{figures_cub/1946/test_explanation_1.png}};
\node at (-8, -2.5) {\includegraphics[height=2cm]{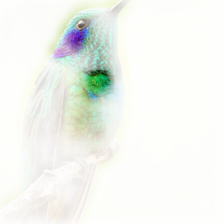}};
\node at (-6, -2.5) {\includegraphics[height=2cm]{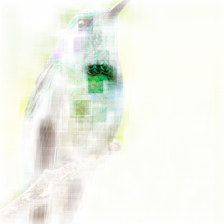}};
\node at (-4, -2.5) {\includegraphics[height=2cm]{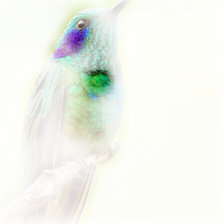}};
\node at (-2, -2.5) {\includegraphics[height=2cm]{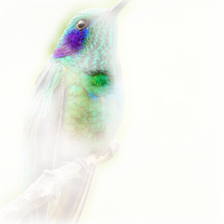}};
\node at (0, -2.5) {\includegraphics[height=2cm]{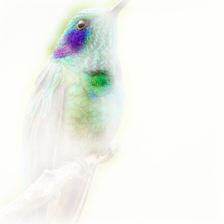}};
\node at (2, -2.5) {\includegraphics[height=2cm]{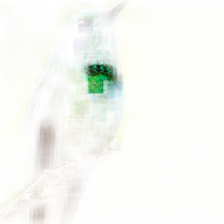}};
\node at (4, -2.5) {\includegraphics[height=2cm]{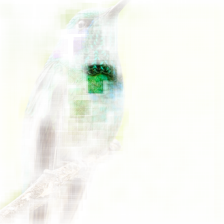}};
\draw[<->, line width=0.5mm] (-10,-3.5) -- (-10,-4);
\draw[<->, line width=0.5mm] (-8,-3.5) -- (-8,-4);
\draw[<->, line width=0.5mm] (-6,-3.5) -- (-6,-4);
\draw[<->, line width=0.5mm] (-4,-3.5) -- (-4,-4);
\draw[<->, line width=0.5mm] (-2,-3.5) -- (-2,-4);
\draw[<->, line width=0.5mm] (0,-3.5)  -- (0,-4);
\draw[<->, line width=0.5mm] (2,-3.5)  -- (2,-4);
\draw[<->, line width=0.5mm] (4,-3.5)  -- (4,-4);

\node at (-3,-6.25) {Class-defining feature interpretations (training images)};
\node at (-10, -5){\includegraphics[height=2cm]{figures_cub/1946/train_explanation_1.png}};
\node at (-8, -5){\includegraphics[height=2cm]{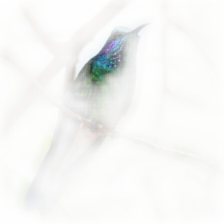}};
\node at (-6, -5){\includegraphics[height=2cm]{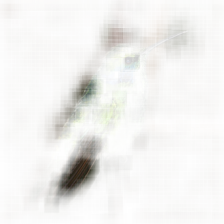}};
\node at (-4, -5){\includegraphics[height=2cm]{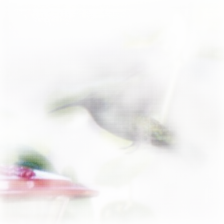}};
\node at (-2, -5){\includegraphics[height=2cm]{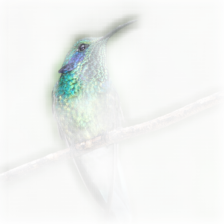}};
\node at (0, -5){\includegraphics[height=2cm]{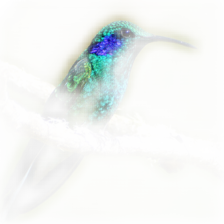}};
\node at (2, -5){\includegraphics[height=2cm]{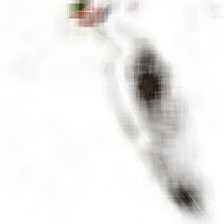}};
\node at (4, -5){\includegraphics[height=2cm]{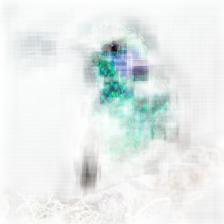}};

\node at (-10,-9) {\makecell{Green\\Violetear}};
\node at (-8,-9) {\makecell{Green\\Violetear}};
\node at (-6,-9) {\makecell{Green\\Violetear}};
\node at (-4,-9) {\makecell{Green\\Violetear}};
\node at (-2,-9) {\makecell{Green\\Violetear}};
\node at (0,-9) {\makecell{Green\\Violetear}};
\node at (2,-9) {\makecell{Pileated\\Woodpecker}};
\node at (4,-9) {\makecell{Green\\Violetear}};

\node at (-3,-9.75) {Training data};

\node at (-10, -7.5) {\includegraphics[height=2cm]{figures_cub/1946/train_image_1.png}};
\node at (-8, -7.5) {\includegraphics[height=2cm]{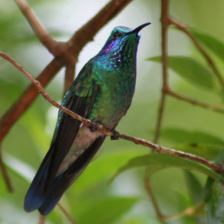}};
\node at (-6, -7.5) {\includegraphics[height=2cm]{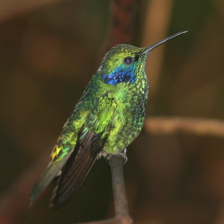}};
\node at (-4, -7.5) {\includegraphics[height=2cm]{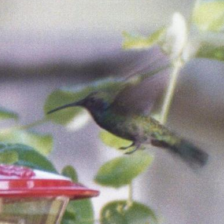}};
\node at (-2, -7.5) {\includegraphics[height=2cm]{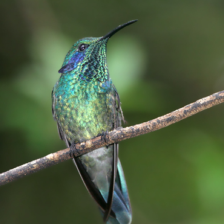}};
\node at (0, -7.5) {\includegraphics[height=2cm]{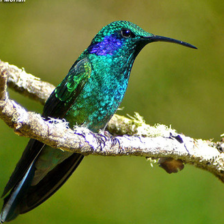}};
\node at (2, -7.5) {\includegraphics[height=2cm]{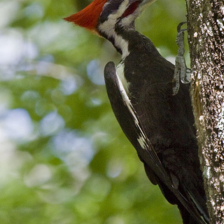}};
\node at (4, -7.5) {\includegraphics[height=2cm]{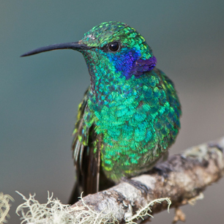}};

\draw [draw=red, line width=0.5mm] (1,-6.5) rectangle (3,-8.5);
\draw [draw=red, line width=0.5mm] (1,-4) rectangle (3,-6);

\end{tikzpicture}
}
\end{subfigure}
\begin{subfigure}{\textwidth}
\resizebox{\textwidth}{!}{
\begin{tikzpicture}[domain=0:15]
\node at (-3,1.25) {Testing image};
\node at (-3, 0) {\includegraphics[height=2cm]{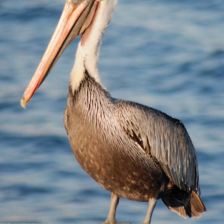}};

\node at (-3,-1.25) {Class-defining feature interpretations (testing image)};
\node at (-10, -2.5) {\includegraphics[height=2cm]{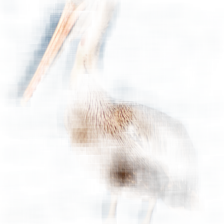}};
\node at (-8, -2.5) {\includegraphics[height=2cm]{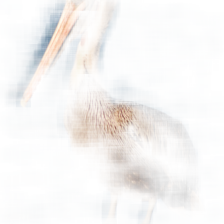}};
\node at (-6, -2.5) {\includegraphics[height=2cm]{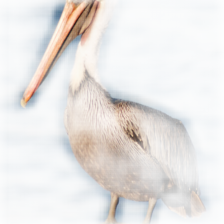}};
\node at (-4, -2.5) {\includegraphics[height=2cm]{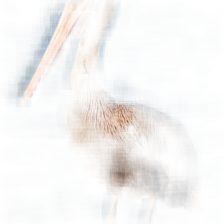}};
\node at (-2, -2.5) {\includegraphics[height=2cm]{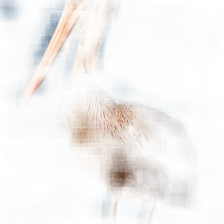}};
\node at (0, -2.5) {\includegraphics[height=2cm]{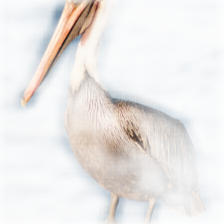}};
\node at (2, -2.5) {\includegraphics[height=2cm]{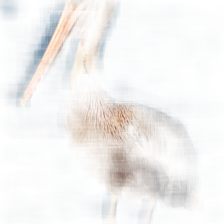}};
\node at (4, -2.5) {\includegraphics[height=2cm]{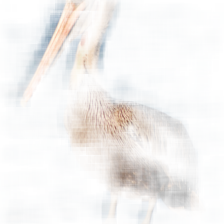}};
\draw[<->, line width=0.5mm] (-10,-3.5) -- (-10,-4);
\draw[<->, line width=0.5mm] (-8,-3.5) -- (-8,-4);
\draw[<->, line width=0.5mm] (-6,-3.5) -- (-6,-4);
\draw[<->, line width=0.5mm] (-4,-3.5) -- (-4,-4);
\draw[<->, line width=0.5mm] (-2,-3.5) -- (-2,-4);
\draw[<->, line width=0.5mm] (0,-3.5)  -- (0,-4);
\draw[<->, line width=0.5mm] (2,-3.5)  -- (2,-4);
\draw[<->, line width=0.5mm] (4,-3.5)  -- (4,-4);

\node at (-3,-6.25) {Class-defining feature interpretations (training images)};
\node at (-10, -5){\includegraphics[height=2cm]{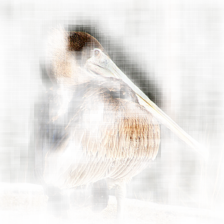}};
\node at (-8, -5){\includegraphics[height=2cm]{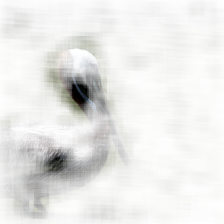}};
\node at (-6, -5){\includegraphics[height=2cm]{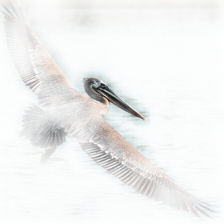}};
\node at (-4, -5){\includegraphics[height=2cm]{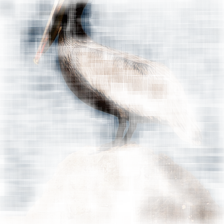}};
\node at (-2, -5){\includegraphics[height=2cm]{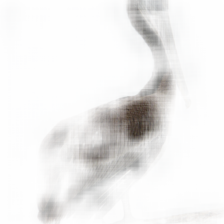}};
\node at (0, -5){\includegraphics[height=2cm]{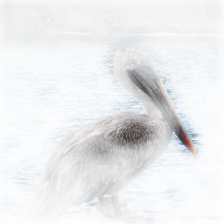}};
\node at (2, -5){\includegraphics[height=2cm]{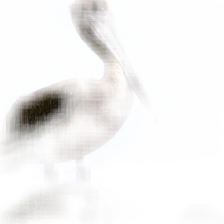}};
\node at (4, -5){\includegraphics[height=2cm]{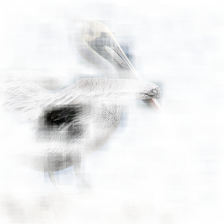}};

\node at (-10,-9.00) {\makecell{Brown\\Pelican}};
\node at (-8,-9.00) {\makecell{Brown\\Pelican}};
\node at (-6,-9.00) {\makecell{Brown\\Pelican}};
\node at (-4,-9.00) {\makecell{Brown\\Pelican}};
\node at (-2,-9.00) {\makecell{Brown\\Pelican}};
\node at (0,-9.00) {\makecell{Brown\\Pelican}};
\node at (2,-9.00) {\makecell{Brown\\Pelican}};
\node at (4,-9.00) {\makecell{Brown\\Pelican}};

\node at (-3,-9.75) {Training data};

\node at (-10, -7.5) {\includegraphics[height=2cm]{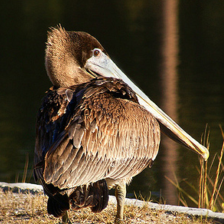}};
\node at (-8, -7.5) {\includegraphics[height=2cm]{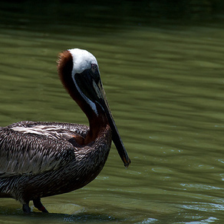}};
\node at (-6, -7.5) {\includegraphics[height=2cm]{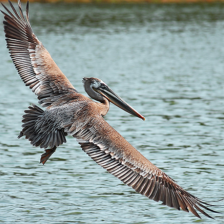}};
\node at (-4, -7.5) {\includegraphics[height=2cm]{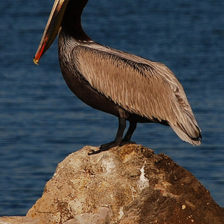}};
\node at (-2, -7.5) {\includegraphics[height=2cm]{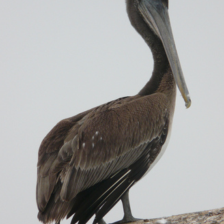}};
\node at (0, -7.5) {\includegraphics[height=2cm]{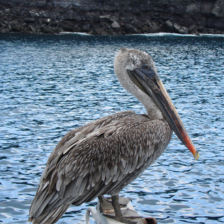}};
\node at (2, -7.5) {\includegraphics[height=2cm]{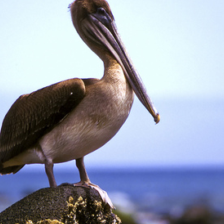}};
\node at (4, -7.5) {\includegraphics[height=2cm]{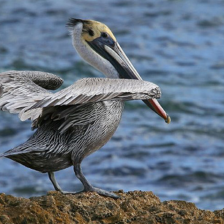}};

\end{tikzpicture}
}
\end{subfigure}

\caption{Additional qualitative results  (CUB-200-2011)}
    \label{qualitative_results_appendix3}
\end{figure}

In Figures \ref{qualitative_results_appendix}, \ref{qualitative_results_appendix2}, \ref{qualitative_results_appendix3}, we show additional qualitative results.

\newpage

\section{Confusion matrices for pseudo-labels}
\label{confusion_matrices_for_pseudolabels_appendix}
In Figure \ref{pseudolabels_vs_real_labels_appendix}, we present the complete confusion matrices for pseudo-labels.
\begin{figure}[h]
    \includegraphics[width=0.22\textwidth]{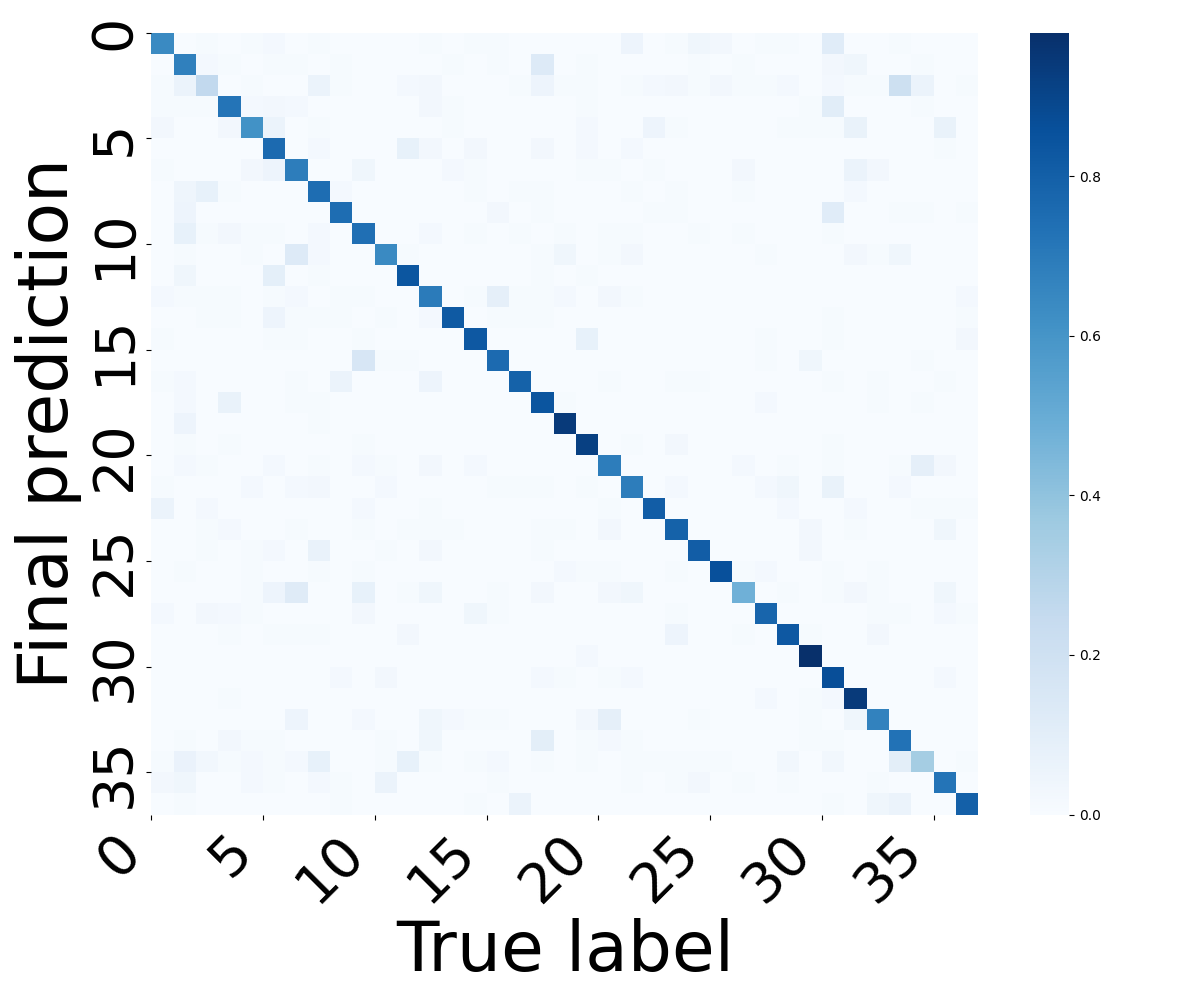}
    \includegraphics[width=0.22\textwidth]{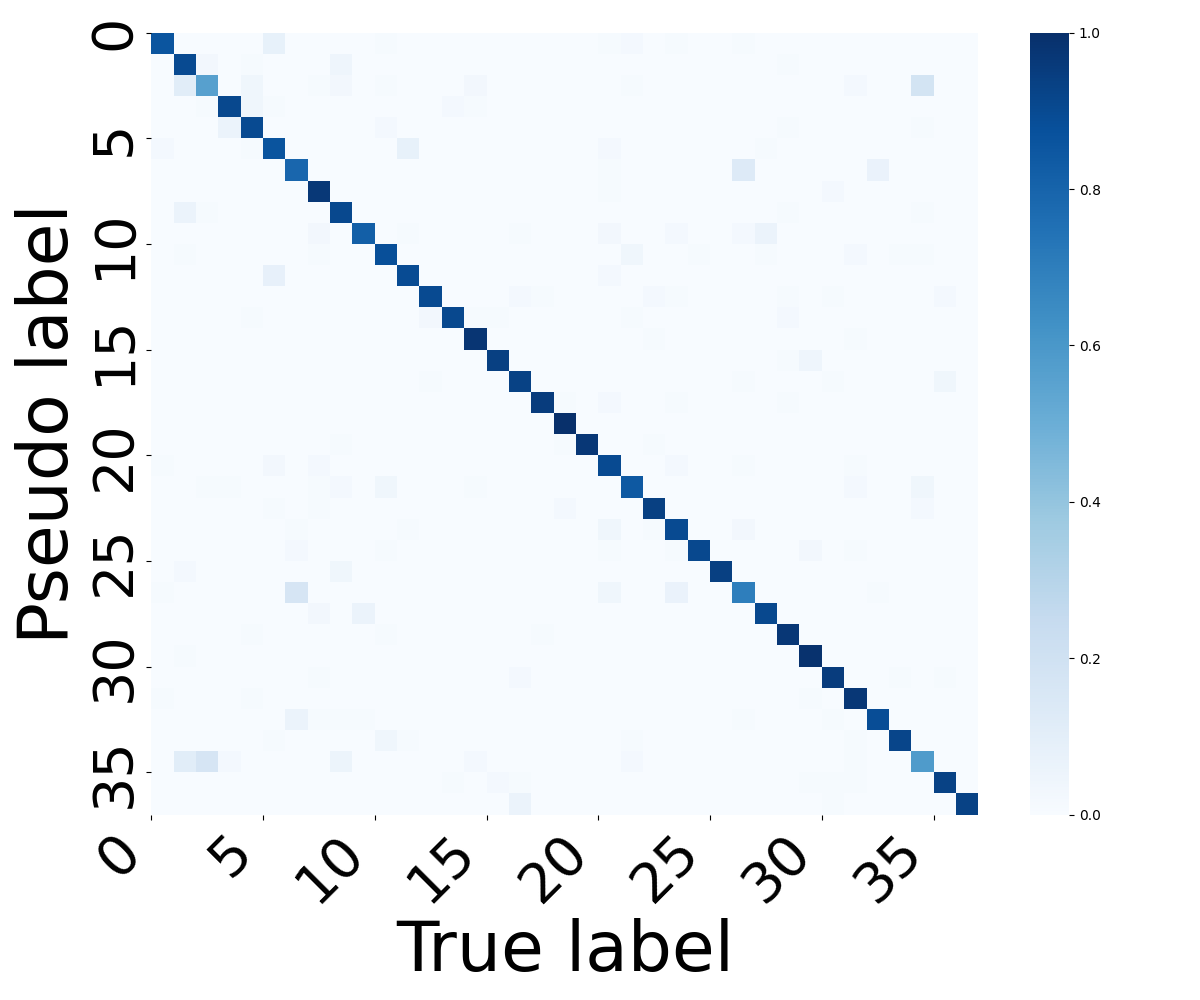}
    \includegraphics[width=0.22\textwidth]{cms/confusion_matrix_3.png}
    \centering
    \caption{Confusion matrices for pseudo-labels: comparison between the ground-truth, final and pseudo-label on Oxford-IIIT Pets dataset}
    \label{pseudolabels_vs_real_labels_appendix}
\end{figure}

\section{Input segmentation}
\label{segmentation_section_appendix}
The images can be segmented according to the dominant feature activated at the pixel level within the input. In Figure \ref{segmentation_results}, we highlight some of the segmentation outputs.

\begin{figure}[h]
    \centering
    \begin{tabular}{c|c} 
        Oxford-IIIT Pets & CUB-200-2011 \\
        \includegraphics[width=0.12\textwidth]{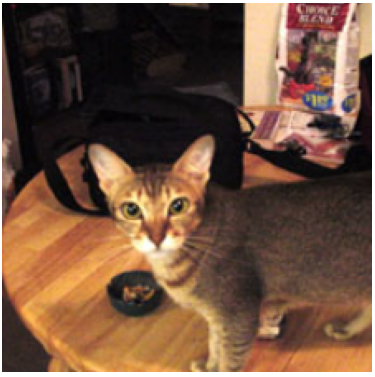}\hspace{0.2cm}\includegraphics[width=0.12\textwidth]{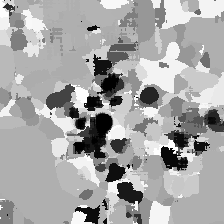}\hspace{0.2cm}\includegraphics[width=0.12\textwidth]{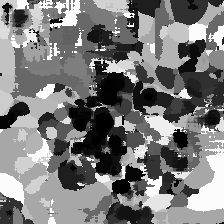} & \includegraphics[width=0.12\textwidth]{figures_cub/1946/test_image.png}\hspace{0.2cm}\includegraphics[width=0.12\textwidth]{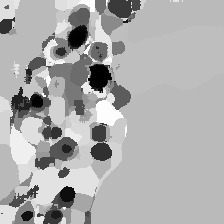}\hspace{0.2cm}\includegraphics[width=0.12\textwidth]{figures_cub/1946/2_seg_image.png} \\
        \includegraphics[width=0.12\textwidth]{figures/698/test_image.png}\hspace{0.2cm}\includegraphics[width=0.12\textwidth]{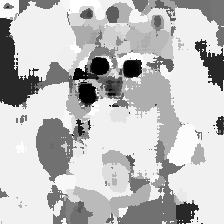}\hspace{0.2cm}\includegraphics[width=0.12\textwidth]{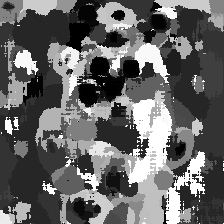} & \includegraphics[width=0.12\textwidth]{figures_cub/2834/test_image.png}\hspace{0.2cm}\includegraphics[width=0.12\textwidth]{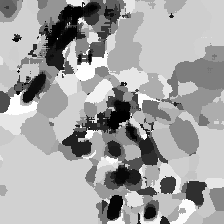}\hspace{0.2cm}\includegraphics[width=0.12\textwidth]{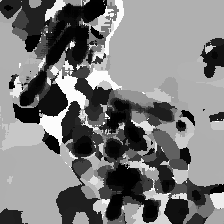} \\
        \includegraphics[width=0.12\textwidth]{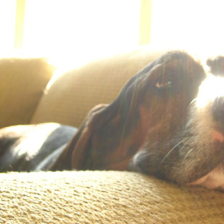}\hspace{0.2cm}\includegraphics[width=0.12\textwidth]{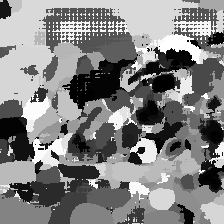}\hspace{0.2cm}\includegraphics[width=0.12\textwidth]{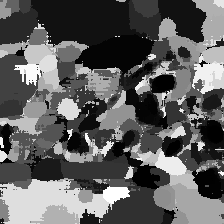} & \includegraphics[width=0.12\textwidth]{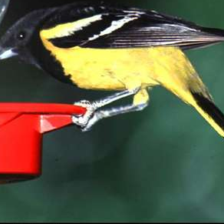}\hspace{0.2cm}\includegraphics[width=0.12\textwidth]{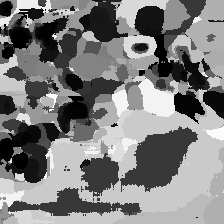}\hspace{0.2cm}\includegraphics[width=0.12\textwidth]{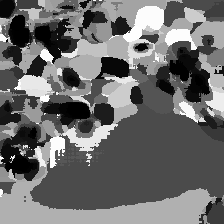}\\
        \includegraphics[width=0.12\textwidth]{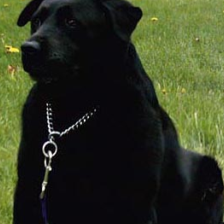}\hspace{0.2cm}\includegraphics[width=0.12\textwidth]{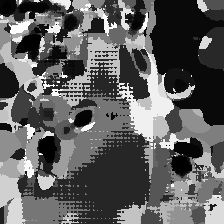}\hspace{0.2cm}\includegraphics[width=0.12\textwidth]{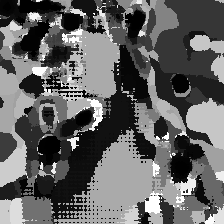} & \includegraphics[width=0.12\textwidth]{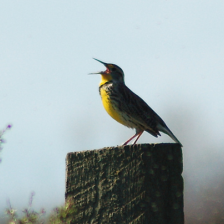}\hspace{0.2cm}\includegraphics[width=0.12\textwidth]{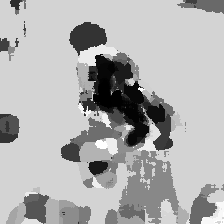}\hspace{0.2cm}\includegraphics[width=0.12\textwidth]{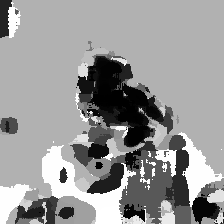}\\
        \includegraphics[width=0.12\textwidth]{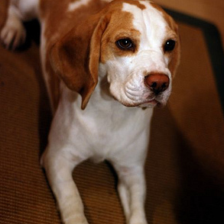}\hspace{0.2cm}\includegraphics[width=0.12\textwidth]{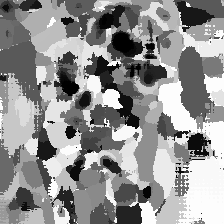}\hspace{0.2cm}\includegraphics[width=0.12\textwidth]{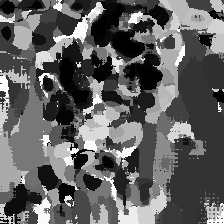} & \includegraphics[width=0.12\textwidth]{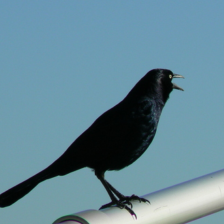}\hspace{0.2cm}\includegraphics[width=0.12\textwidth]{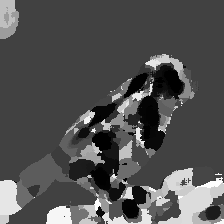}\hspace{0.2cm}\includegraphics[width=0.12\textwidth]{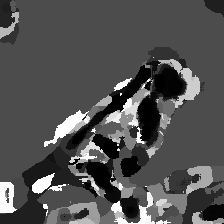}
    \end{tabular}
    \caption{Image segmentation results: testing image, segmentations by leading CDFs, and segmentations by all features}
    \label{segmentation_results}
\end{figure}

\newpage

\section{Broader impacts}
\label{broader_impacts}
 While some of the existing \textit{post hoc} explanation methods can explain the decision making, they do not follow the original decision making process \cite{rudin2019stop}. This, therefore, cannot satisfy the current legal, ethical and policy-making needs. In contrast, by-design methods provide explanations which are causally linked with the decision making process. Such alternative is especially important for safety-critical applications, such as autonomous driving, robotics, medical imagery. 
 
 As finetuning-free learning was not considered the primary goal of this work, it was merely documented and not investigated further. It remains to be seen as to why \methodname\ results in surprisingly good finetuning-free performance.

\section{Limitations}
\label{limitations_section}
Use of the pseudo-labels for preliminary selection of features can be also considered as a limitation, which is common for other works using concept-based interpretations due to the fact that feature selection necessitates pre-selection of the proposal class for subsequent refinement. \cite{tan2024post} describes the similar problem for their \textit{post hoc} analysis method as a feature refinement problem.

\end{document}